\documentclass[authordate,onecolumn]{miri-tech-article}

\usepackage{miritools}
\usepackage{stmaryrd}
\usepackage[utf8]{inputenc}
\usepackage[english]{babel}
\usepackage{csquotes}
\usepackage[export]{adjustbox} %image frame colours
\definecolor{r}{RGB}{197, 46, 14}
\definecolor{o}{RGB}{197, 119, 14}
\definecolor{b}{RGB}{45, 14, 197}
\definecolor{g}{RGB}{92, 180, 32}

\usepackage{lipsum}

\title{Temporal Inference with Finite Factored Sets}

\author{Scott Garrabrant \\ Machine Intelligence Research Institute \\ scott@intelligence.org}

\newtheorem{example}{Example}
\newtheorem{conjecture}{Conjecture}
\newtheorem{proposition}{Proposition}
\newtheorem{corollary}{Corollary}
\newtheorem{lemma2}{Lemma}
\begin{document}

%\nocopyright

\newcommand{\du}[1]{\bigsqcup(#1)}
\newcommand{\pr}[1]{\bigsqcap(#1)}
\newcommand{\ortho}[3][F]{#2\mathbin{\perp^{#1}}#3}
\newcommand{\co}[4][F]{#2\mathbin{\perp^{#1}}#3\mid #4}
\newcommand{\cod}[4][D]{#2\mathbin{\perp_{#1}}#3\mid #4}
\newcommand{\ncod}[4][D]{#2\mathbin{\rightleftharpoons_{#1}}#3\mid #4}
\newcommand{\coc}[5][F]{#2\mathbin{\perp^{#1}_{#5}}#3\mid #4}
\newcommand{\parts}[1][S]{\text{Part}(#1)}
\maketitle
\begin{abstract}
We propose a new approach to temporal inference, inspired by the Pearlian causal inference paradigm---though quite different from Pearl's approach formally. Rather than using directed acyclic graphs, we make use of \emph{factored sets}, which are sets expressed as Cartesian products. We show that finite factored sets are powerful tools for inferring temporal relations. We introduce an analog of $d$-separation for factored sets, \emph{conditional orthogonality}, and we demonstrate that this notion is equivalent to conditional independence in all probability distributions on a finite factored set.
\end{abstract}

\tableofcontents

\vspace{25mm}

\section{Introduction}
\subsection{Pearlian Causal Inference}
Judea Pearl's theory of inferred causation (e.g., as presented in chapter 2 of \emph{Causality: Models, Reasoning, and Inference}) was a deep advance in our understanding of the nature of time. The Pearlian paradigm allows us to infer causal relationships between variables using statistical data, and thereby infer temporal sequence---in defiance of the old adage that correlation does not imply causation.\nocite{Pearl:2000}

In particular, given \emph{a collection of variables} and \emph{a joint probability distribution over those variables}, the Pearlian paradigm can often infer temporal relationships between the variables.

The joint probability distribution is usually what gets emphasized in discussions of Pearl's approach. Quite a bit of work is being done, however, by the assumption that we are handed ``a collection of variables'' to reason about. The Pearlian paradigm is not inferring temporal relationships from purely statistical data, but rather inferring temporal relationships from statistical data together with data about how to factorize the world into variables.\footnote{Although I say ``factorize'' here, note that this will not be the kind of factorization that shows up in finite factored sets, because (as we will see) disjoint factors must be independent in a finite factored set. I appeal to the same concept in both contexts because factorization is just a very general and useful concept, rather than to indicate a direct connection.}

A doctor who misdiagnoses their patient or misidentifies a symptom may base their subsequent reasoning on a wrong factorization of the situation into causally relevant variables. We would ideally like to build fewer assumptions like this into our model of inference, and instead allow the reasoner to figure such facts out, consider the merits of different factorizations into variables, etc.

Instead of beginning with a collection of variables and a joint probability distribution over those variables, one could imagine starting with just a finite sample space and a probability distribution on that sample space. In this way, we might hope to do temporal inference purely using statistical data, without relying on \emph{a priori} knowledge of a canonical way of factoring the situation into variables.

How might one do temporal inference without an existing factorization? One way might be to just consider all possible variables that can be defined on the sample space. This gives us one variable for each partition of the set.

However, when one tries to apply Pearl's methods to this collection of variables, one quickly runs into a problem: many of the variables definable on a fixed set are deterministic functions of each other. The Pearlian paradigm, as presented in the early chapters of \emph{Causality}, lacks tools for performing temporal inference on variables that are highly deterministically related.\footnote{At least, it lacks such causal inference tools unless we assume access to interventional data.}

We will introduce a new approach to temporal inference instead---one which is heavily inspired by the Pearlian paradigm, but approaches the problem with a very different formal apparatus, and does not make use of graphical models.

\subsection{Overview}

We'll begin by introducing the concept of a finite factored set, in Section 2. This will be our analogue of the directed acyclic graphs in Pearl's framework.

In Section 3, we will introduce the concepts of time and orthogonality, which can be read off of a finite factored set. In Pearl's framework, ``time'' corresponds to directed paths between nodes, and ``orthogonality'' corresponds to nodes that have no common ancestor.

In Section 4, we will introduce conditional orthogonality, which is our analogue of $d$-separation. We show that conditional orthogonality satisfies (a modified version of) the compositional semigraphoid axioms. We then (in Section 5) prove the fundamental theorem of finite factored sets, which states that conditional orthogonality is equivalent to conditional independence in all probability distributions on the finite factored set.

In Section 6, we discuss how to do temporal inference using finite factored sets, and give two examples. Finally, in Section 7 we discuss applications and future work, with an emphasis on temporal and conceptual inference, generalizing finite factored sets to the infinite case, and applications to embedded agency \citep{Demski:2019:embedded}.

And here, we take our leave of Pearl. We've highlighted this approach's relationship to the Pearlian paradigm in order to motivate finite factored sets and explain how we'll be using them in this paper. Technically, however, our approach is quite unlike Pearl's, and the rest of the paper will stand alone.

\section{Factorization}
Before giving a definition of finite factored sets, we will recall the definition of a partition, and give some basic notation related to partitions.

We do this for two reasons. First, we will use partitions in the definition of a factored set; and second, we want to draw attention to a duality between the notion of a partition, and the notion of a factorization.
\subsection{Partitions}
We begin with a definition of disjoint union.

\begin{definition}[disjoint union]
Given a set $S$ of sets, let $\du{S}$ denote the set of all ordered pairs $(T,t)$, where $T\in S$ and $t\in T$.\footnote{Note that this definition and Definition \ref{proddef} could have been made more general by taking $S$ to be a multiset.}
\end{definition}
\begin{definition}[partition]
A partition of a set $S$ is a set $X\subseteq \mathcal{P}(S)$ of nonempty subsets of $S$ such that the function $\iota:\du{X}\rightarrow S$ given by $\iota(x,s)=s$ is a bijection.\footnote{$\mathcal{P}(S)$ denotes the power set of $S$.}

Let $\parts$ denote the set of all partitions of $S$. The elements of a partition are called parts.
\end{definition}

An equivalent definition of partition is often given: a partition is a set $X$ of nonempty subsets of $S$ that are pairwise disjoint and union to $S$. We choose the above definition because it will make the symmetry between partitions and factorizations more obvious.

\begin{definition}[trivial partition]
A partition $X$ of a set $S$ is called trivial if $|X|=1$.
\end{definition}
\begin{definition}
Given a partition $X$ of a set $S$, and an element $s\in S$, let $[s]_X$ denote the unique $x\in X$ such that $s\in x$. 
\end{definition}
\begin{definition}
Given a partition $X$ of a set $S$, and elements $s_0,s_1\in S$, we say $s_0\sim_X s_1$ if $[s_0]_X=[s_1]_X$.
\end{definition}
\begin{proposition}
Given a partition $X$ of a set $S$, $\sim_X$ is an equivalence relation on $S$.
\end{proposition}
\begin{proof}
Trivial.
\end{proof}
\begin{definition}[finer and coarser]
We say that a partition $X$ of $S$ is finer than another partition $Y$ of $S$, if for all $s_0,s_1\in S$, if $s_0\sim_{X}s_1$, then $s_0\sim_{Y} s_1$.

If $X$ is finer than $Y$, we also say $Y$ is coarser than $X$, and we write $X\geq_S Y$ and $Y\leq_S X$.
\end{definition}
\begin{definition}[discrete and indiscrete partitions]
Given a set $S$, let $\text{Dis}_S=\{\{s\}\mid s\in S\}$.

If $S$ is empty, let $\text{Ind}_S=\{\}$, and if $S$ is nonempty, let $\text{Ind}_S=\{S\}$.

$\text{Dis}_S$ is called the discrete partition, and $\text{Ind}_S$ is called the indiscrete partition.
\end{definition}
\begin{proposition}
For any set $S$, $\geq_S$ is a partial order on $\parts$. Further, for all $X\in\parts$, $\text{Dis}_S\geq_S X$ and $X\geq_S \text{Ind}_S$.
\end{proposition}
\begin{proof}
Trivial.
\end{proof}
While both notations are sometimes used, it is more standard to draw the symbol in the opposite direction and have $X\leq Y$ when $X$ is finer than $Y$. We choose to go against that standard because we want to think of partitions in part as the ability to distinguish between elements, and finer partitions correspond to greater ability to distinguish.\footnote{In our view, ``$Y \geq X$'' is also a more natural way to visually represent a mapping between a three-part partition $Y$ that is finer than a two-part partition $X$.}
\begin{definition}[common refinement]
Given a set $C$ of partitions of a fixed set $S$, let $\bigvee_S(C)$ denote the partition $X\in\parts$ satisfying $s_0\sim_{X} s_1$ if and only if $s_0\sim_c s_1$ for all $c\in C$.
Given $X,Y\in\parts$, we let $X\vee_S Y=\bigvee_S(\{X,Y\})$.
\end{definition}

\subsection{Factorizations}
We start with a definition of Cartesian product.
\begin{definition}[Cartesian product]\label{proddef}
Given a set $S$ of sets, let $\pr{S}$ denote the set of all functions $f:S\rightarrow \du{S}$ such that for all $T\in S$, $f(T)$ is of the form $(T,t)$, for some $t\in T$.
\end{definition}
We can now give the definition of a factorization of a set.
\begin{definition}[factorization]
A factorization of a set $S$ is a set $B\subseteq\parts$ of nontrivial partitions of $S$ such that the function $\pi:S\rightarrow\pr{B}$, given by $\pi(s)=(b\mapsto(b,[s]_b))$, is a bijection.

Let $\text{Fact}(S)$ denote the set of all factorizations of $S$. The elements of a factorization are called factors.
\end{definition}
In other words, a set of nontrivial partitions is a factorization of $S$ if for each way of choosing one part from each factor, there exists a unique element of $S$ in the intersection of those parts.

Notice the duality between the definitions of partition and factorization. We replace subsets with partitions, nonempty with nontrivial, and disjoint union with Cartesian product, and we reverse the direction of the function. We can think of a factorization of $S$ as a way to view $S$ as a product, in the same way that a partition was a way to view $S$ as a disjoint union.

A factored set is just a set together with a factorization of that set.
\begin{definition}[factored set]
A factored set $F$ is an ordered pair $(S,B)$, such that $B$ is a factorization of $S$. 

If $F=(S,B)$ is a factored set, we let $\text{set}(F)=S$, and let $\text{basis}(F)=B$.
\end{definition}

\begin{proposition}\label{templabel1}
Given a factored set $F=(S,B)$, and elements $s_0,s_1\in S$, if $s_0\sim_b s_1$ for all $b\in B$, then $s_0=s_1$.

\end{proposition}
\begin{proof}
Let $F=(S,B)$ be a finite factored set, and let $s_0,s_1\in S$ satisfy $s_0\sim_b s_1$ for all $b\in B$. 

Let $\pi:S\rightarrow \pr{B}$ be given by $\pi(s)=(b\mapsto(b,[s]_b))$, as in the definition of factorization. Then $\pi(s_0)=(b\mapsto(b,[s_0]_b))=(b\mapsto(b,[s_1]_b))=\pi(s_1)$. Since $\pi$ is bijective, this means $s_0=s_1$.
\end{proof}

\subsection{Chimera Functions}
The following theorem can be viewed as an alternate characterization of factorization. We will use this alternate characterization to define chimera functions, which will be useful tools for manipulating elements of factored sets.
\begin{theorem}\label{templabel2}
Given a set $S$, a set $B$ of nontrivial partitions of $S$ is a factorization of $S$ if and only if for every function $g:B\rightarrow S$, there exists a unique $s\in S$ such that for all $b\in B$, $s\sim_b g(b)$.
\end{theorem}
\begin{proof}
First, we let $B$ be a factorization of $S$, and let $g:B\rightarrow S$ be any function. We want to show that there exists a unique $s\in S$ such that for all $b\in B$, $s\sim_b g(b)$. Let $\pi:S\rightarrow \pr{B}$ be given by $\pi(s)=(b\mapsto(b,[s]_b))$, as in the definition of factorization. Note that $\pi$ is bijective, and thus has an inverse.

Let $s=\pi^{-1}(b\mapsto (b,[g(b)]_b))$. Observe that this is well-defined, because $(b\mapsto (b,[g(b)]_b))$ is in fact in $\pr{B}$. We will show that $s\sim_b g(b)$ for all $b\in B$, and the uniqueness of this $s$ will then follow directly from Proposition \ref{templabel1}. 

We have $\pi(s)=(b\mapsto[s]_b)$ by the definition of $\pi$. However, we also have $\pi(s)=(b\mapsto[g(b)]_b)$ by the definition of $s$. Thus, $b\mapsto[s]_b$ and $b\mapsto[g(b)]_b$ are the same function, so $[s]_b=[g(b)]_b$ for all $b\in B$, so $s\sim_b g(b)$ for all $b\in B$.

Conversely, let $S$ be any set, and let $B$ be any set of nontrivial partitions of $S$. Assume that for all $g:B\rightarrow S$, there exists a unique $s\in S$ satisfying $s\sim_b g(b)$ for $b\in B$. Again, let $\pi:S\rightarrow \pr{B}$ be given by $\pi(s)=(b\mapsto(b,[s]_b))$, as in the definition of factorization. We want to show that $\pi$ is invertible.

First, we show that $\pi$ is injective. Take an arbitrary $s_0\in S$, and let $g:B\rightarrow S$ be the constant function satisfying $g(b)=s_0$ for all $b\in B$. Given another $s_1\in S$, if $\pi(s_0)=\pi(s_1)$, then $(b\mapsto[s_0]_b)=(b\mapsto[s_1]_b)$, so $[s_1]_b=[s_0]_b=[g(b)]_b$ for all $b\in B$, so $s_0\sim_b s_1\sim_b g(b)$ for all $b\in B$. Since there is a unique $s\in S$ satisfying $s\sim_b g(b)$ for all $b\in B$, this means $s_0=s_1$. Thus $\pi$ is injective.

To see that $\pi$ is surjective, consider some arbitrary $h\in \pr{B}$. We want to show that there exists an $s\in S$ with $h=\pi(s)$.

For all $b\in B$, let $H_b\in b$ be given by $h(b)=(b,H_b)$, which is well-defined since $h\in \pr{B}$. Note that $H_b$ is a nonempty subset of $S$, so there exists a function $g:B\rightarrow S$ with $g(b)\in H_b$ for all $b\in B$. Fix any such $g$, and let $s$ satisfy $s\sim_b g(b)$ for all $b\in B$. 

We thus have that for all $b\in B$, $h(b)=(b,H_b)=(b,[g(b)]_b)=(b,[s]_b)=\pi(s)(b)$, so $h=\pi(s)$. Thus $\pi$ is surjective. 

Since $\pi$ is bijective, we have that $B$ is a factorization of $S$.
\end{proof}
This also gives us that factors are disjoint from each other.
\begin{corollary}\label{basisdisjoint}
Given a factored set $F=(S,B)$ and distinct factors $b_0,b_1\in B$, $b_0\cap b_1=\{\}$.
\end{corollary}
\begin{proof}
Assume by way of contradiction that $T\in b_0\cap b_1$. Since $b_0$ is nontrivial, there must be some other $T^\prime\in b_0$ with $T\cap T^\prime=\{\}$. Let $g:B\rightarrow S$ be any function such that $g(b_0)\in T^\prime$ and $g(b_1)\in T$. Then there can be no $s$ such that $s\sim_{b_0} g(b_0)$ and $s\sim_{b_1}g(b_1)$, since then $s$ would be in both $T$ and $T^\prime$. This contradicts Theorem \ref{templabel2}.
\end{proof}

We are now ready to define the chimera function of a factored set.

\begin{definition}[chimera function]
Given a factored set $F=(S,B)$, the chimera function (of $F$) is the function $\chi^F:(B\rightarrow S)\rightarrow S$ defined by $\chi^F(g)\sim_b g(b)$ for all $g:B\rightarrow S$ and $b\in B$.
\end{definition}
The name ``chimera function'' comes from the fact that $\chi^F$ can be viewed as building an element of $S$ by fusing together the properties of various different elements. Since we will often apply the chimera function to functions $g$ that only take on two values, we will give notation for this special case.
\begin{definition}
Given a factored set $F=(S,B)$, and a subset $C\subseteq B$, let $\chi^F_C:S\times S\rightarrow S$ be given by $\chi^F_C(s,t)=\chi^F(g)$, where $g:B\rightarrow S$ is given by $g(b)=s$ if $b\in C$, and $g(b)=t$ otherwise.

For $T,R\subseteq S$, we will write $\chi^F_C(T,R)$ for $\{\chi^F_C(t,r)\mid t\in T,\ r\in R\}$.
\end{definition}
The following is a list of properties of $\chi^F_C$, which will be useful in later proofs. All of these properties follow directly from the definition of $\chi^F_C$.
\begin{proposition}
Fix $F=(S,B)$, a factored set, $C,D\subseteq B$, and $s,t,r\in S$.

\begin{enumerate}
    \item $\chi^F_C(s,t)\sim_c s$ for all $c\in C$.
    \item $\chi^F_C(s,t)\sim_b t$ for all $b\in B\setminus C$.
    \item $\chi^F_{C}(s,s)=s$.
    \item $\chi^F_{B\setminus C}(s,t)=\chi^F_{C}(t,s)$.
    \item $\chi^F_{C\cup D}(s,t)=\chi^F_{C}(s,\chi^F_{D}(s,t))$.
    \item $\chi^F_{C\cap D}(s,t)=\chi^F_{C}(\chi^F_{D}(s,t),t)$.
    \item $\chi^F_{C}(\chi^F_{C}(s,t),r)=\chi^F_{C}(s,\chi^F_{C}(t,r))=\chi^F_{C}(s,r)$.
    \item $\chi^F_{C}(s,\chi^F_{D}(t,r))=\chi^F_{D}(\chi^F_{C}(s,t),\chi^F_{C}(s,r))$.
    \item $\chi^F_{C}(\chi^F_{D}(s,t),r)=\chi^F_{D}(\chi^F_{C}(s,r),\chi^F_{C}(t,r))$.
    \item $\chi^F_B(s,t)=s$.
    \item $\chi^F_{\{\}}(s,t)=t$.
    
\end{enumerate}
\end{proposition}
\begin{proof}
Trivial.
\end{proof}
\subsection{Trivial Factorizations}
We now define a notion of a trivial factorization of a set, and show that every set has a unique trivial factorization.
\begin{definition}[trivial factorization]
A factorization $B$ of a set $S$ is called trivial if $|B|\leq 1$. A factored set $(S,B)$ is called trivial if $B$ is trivial.
\end{definition}
\begin{proposition}\label{templabel4}
For every set $S$, there exists a unique trivial factorization $B$ of $S$. If $|S|\neq 1$, this trivial factorization is given by $B=\{\text{Dis}_S\}$, and if $|S|=1$, it is given by $B=\{\}$.
\end{proposition}
\begin{proof}
We start with the case where $|S|=0$. The only partition of $S$ is $\{\}$, so we only need to consider the sets of partitions $\{\{\}\}$ and $\{\}$ as potential factorizations. $\{\{\}\}$ is vacuously a factorization of $S$ by Theorem \ref{templabel2}, since there are no functions from $\{\{\}\}$ to $S$. $\{\}$ is not a factorization by Theorem \ref{templabel2}, since there is a function from $\{\}$ to $S$, but there is no element of $S$. Thus, when $|S|=0$, $\{\{\}\}=\{\text{Dis}_S\}$ is the unique trivial factorization of $S$.

Next, consider the case where $|S|=1$. First, observe that the unique $s\in S$ vacuously satisfies $s\sim_b g(b)$ for all $g:\{\}\rightarrow S$ and $b\in \{\}$, since there is no $b\in \{\}$. Thus, by Theorem \ref{templabel2}, $\{\}$ is a factorization of $S$. Further, $\{\}$ is the only factorization of $S$, since there are no nontrivial partitions of $S$. Thus, when $|S|=1$, $\{\}$ is the unique trivial factorization of $S$.

Next, we consider the case where $|S|\geq 2$. Observe that $\text{Dis}_S$ is a nontrivial partition of $S$. Let $B=\{\text{Dis}_S\}$. We want to show that $B$ is a factorization of $S$. By Theorem \ref{templabel2}, it suffices to show that for all $g:B\rightarrow S$, there exists a unique $s\in S$ with $s\sim_{\text{Dis}_S} g(\text{Dis}_S)$. We can take $s=g(\text{Dis}_S)$, which clearly satisfies $s\sim_{\text{Dis}_S} g(\text{Dis}_S)$. This $s$ is unique, since if $s^\prime\sim_{\text{Dis}_S} g(\text{Dis}_S)$, then $s^\prime\in[g(\text{Dis}_S)]_{\text{Dis}_S}=[s]_b=\{s\}$, so $s^\prime = s$. Thus $B$ is a factorization of $S$.

On the other hand, if $|S|\geq 2$, $\{\}$ is not a factorization of $S$, since if it were, Proposition \ref{templabel1} would imply that all elements of $S$ are equal. Further, for any partition $b$ of $S$, with $b\neq \{\text{Dis}_S\}$, there must exist $s_0,s_1\in S$, with $s_0\sim_b s_1$, but $s_0\neq s_1$. Thus $\{b\}$ cannot be a factorization of $S$ by Proposition \ref{templabel1}. Thus when $|S|\geq 2$, $\text{Dis}_S$ is the unique trivial factorization of $S$.
\end{proof}
\subsection{Finite Factored Sets}
This paper will primarily be about finite factored sets.
\begin{definition}
If $F=(S,B)$ is a factored set, the size of $F$, written $\text{size}(F)$, is the cardinality of $S$. The dimension of $F$, written $\text{dim}(F)$, is the cardinality of $B$. $F$ is called finite if its size is finite, and finite-dimensional if its dimension is finite.
\end{definition}
We suspect that the theory of infinite factored sets is both interesting and important. However, it is outside of the scope of this paper, which will require finiteness for many of its key results.

Some of the definitions and results in this paper will be given for finite factored sets, in spite of the fact that they could easily be extended to finite-dimensional or arbitrary factored sets. This is because they can often be extended in more than one way, and determining which extension is most natural requires further developing the theory of arbitrary factored sets.

\begin{proposition}
Every finite factored set is also finite-dimensional.
\end{proposition}
\begin{proof}
If $F=(S,B)$ is a factored set, $B$ is a set of sets of subsets of $S$. Thus, $|B|\leq 2^{2^{|S|}}$.
\end{proof}
This bound is horrible and will be improved in Proposition \ref{templabel3}. First, however, we will take a look at the number of factorizations of a fixed finite set. 

\begin{proposition}\label{sizeffs}
Let $F=(S,B)$ be a finite factored set. Then $|S|=\prod_{b\in B} |b|$.
\end{proposition}

\begin{proof}
Trivial.
\end{proof}

\begin{proposition}\label{templabel5}
If $|S|$ is equal to $0$, $1$, or a prime, the trivial factorization of $S$ is the only factorization of $S$.
\end{proposition}
\begin{proof}
If $|S|=0$ or $|S|=1$, then $|\parts|=1$, so $B\subseteq\parts$ can have cardinality at most 1.

If $|S|=p$, a prime, then by Proposition \ref{sizeffs}, $|b|$ must divide $p$ for all $b\in B$. Since factorizations cannot contain trivial partitions, this means $|b|=p$ for all $b\in B$. However, $\{\{s\}\mid s\in S\}$ is the only element of $\parts$ of cardinality $p$, so $|B|\leq 1$.
\end{proof}
On the other hand, in the case where $|S|$ is finite and composite, the number of factorizations of $S$ grows very quickly, as seen in Table \ref{tab:my_label}.
\begin{table}[ht]
    \centering
    \begin{tabular}{|r|l||r|l|}
    \hline
        $|S|$ & $|\text{Fact}(S)|$ & $|S|$  & $|\text{Fact}(S)|$ \\
    \hline
        0 & 1 & 13 & 1\\
        1 & 1 & 14 & 8648641\\
        2 & 1 & 15 & 1816214401\\
        3 & 1 & 16 & 181880899201\\
        4 & 4 & 17 & 1\\
        5 & 1 & 18 & 45951781075201\\
        6 & 61 & 19 & 1\\
        7 & 1 & 20 & 3379365788198401\\
        8 & 1681 & 21 & 1689515283456001\\
        9 & 5041 & 22 & 14079294028801\\
        10 & 15121 & 23 & 1\\
        11 & 1 & 24 & 4454857103544668620801\\
        12 & 13638241 & 25 & 538583682060103680001\\
    \hline
    \end{tabular}
    \caption{The number of factorizations of a set $S$ with cardinality up to 25.}
    \label{tab:my_label}
\end{table}
\linebreak
\noindent
Given the naturalness of the notion of factorization, we were surprised to discover that this sequence did not exist on the On-Line Encyclopedia of Integer Sequences (OEIS). We added the sequence, \href{https://oeis.org/A338681}{A338681}, on April 30, 2021.\nocite{OEIS:2021}

To give one concrete example, the four factorizations of  the set $\{0,1,2,3\}$ are:
\begin{itemize}
    \item $\{\{\{0\},\{1\},\{2\},\{3\}\}\}$,
    \item $\{\{\{0,1\},\{2,3\}\},\{\{0,2\},\{1,3\}\}\}$,
    \item 
    $\{\{\{0,1\},\{2,3\}\},\{\{0,3\},\{1,2\}\}\}$, and
    \item
    $\{\{\{0,2\},\{1,3\}\},\{\{0,3\},\{1,2\}\}\}$.
\end{itemize}
\begin{proposition}\label{templabel3}
Let $F$ be a finite factored set. 
\begin{enumerate}
    \item If $\text{size}(F)=0$, then $\text{dim}(F)=1$.
    \item If $\text{size}(F)=1$, then $\text{dim}(F)=0$.
    \item If $\text{size}(F)=p$ is prime, then $\text{dim}(F)=1$.
    \item If $\text{size}(F)=p_0\dots p_{k-1}$ is a product of $k\geq 2$ primes, then $1\leq \text{dim}(F)\leq k$.
\end{enumerate}
\end{proposition}
\begin{proof}
The first three parts follow directly from Proposition \ref{templabel4} and Proposition \ref{templabel5}. For the fourth part, let $F=(S,B)$, and let $|S|=p_0\dots p_{k-1}$ be a product of $k\geq 2$ primes.

By Proposition \ref{sizeffs}, $|S|=\prod_{b\in B} |b|$. Consider an arbitrary $b\in B$. Since $b$ is a nontrivial partition of a finite set $S$, $|b|$ is finite and $|b|\neq 1$. If $|b|$ were $0$, then $|S|$ would be $0$. Thus $|b|$ is a natural number greater than or equal to 2. $B$ cannot be empty, since $|S|\neq 1$. If $|B|$ were greater than $k$, then we would be able to express $|S|$ as a product of more than $k$ natural numbers greater than or equal to $2$, which is clearly not possible since $|S|$ is a product of $k$ primes. Thus $1\leq\text{dim}(F)\leq k$.
\end{proof}

\section{Orthogonality and Time}
The main way we'll be using factored sets is as a foundation for talking about concepts like orthogonality and time. Finite factored sets will play a role that's analogous to that of directed acyclic graphs in Pearlian causal inference.

To utilize factored sets in this way, we will first want to introduce the concept of generating a partition with factors.
\subsection{Generating a Partition with Factors}
\begin{definition}[generating a partition]\label{templabel9}
Given a finite factored set $F=(S,B)$, a partition $X\in \parts$, and a $C\subseteq B$, we say $C$ generates $X$ (in $F$), written $C\vdash^F X$, if $\chi^F_C(x,S)=x$ for all $x\in X$.
\end{definition}
The following proposition gives many equivalent definitions of $\vdash^F$.
\begin{proposition}\label{equivgen}
Let $F=(S,B)$ be a finite factored set, let $X\in \parts$ be a partition of $S$, and let $C$ be a subset of $B$. The following are equivalent:
\begin{enumerate}
    \item $C\vdash^F X$.
    \item $\chi^F_C(x,S)=x$ for all $x\in X$.
    \item $\chi^F_C(x,S)\subseteq x$ for all $x\in X$.
    \item $\chi^F_C(x,y)\subseteq x$ for all $x,y\in X$.
    \item $\chi^F_C(s,t) \in [s]_X$ for all $s,t\in S$.
    \item $\chi^F_C(s,t)\sim_X s$ for all $s,t\in S$.
    \item $X\leq_S\bigvee_S(C)$.
\end{enumerate}
\end{proposition}
\begin{proof}
The equivalence of conditions 1 and 2 is by definition.

The equivalence of conditions 2 and 3 follows directly from the fact that $\chi^F_C(s,s)=s$ for all $s\in x$, so $\chi^F_C(x,S)\supseteq \chi^F_C(x,x)\supseteq x$.

To see that conditions 3 and 4 are equivalent, observe that since $S=\bigcup_{y\in X} y$, $\chi^F_C(x,S)=\bigcup_{y\in X}\chi^F_C(x,y)$. Thus, if $\chi^F_C(x,S)\subseteq x$, $\chi^F_C(x,y)\subseteq x$ for all $y\in X$, and conversely if $\chi^F_C(x,y)\subseteq x$ for all $y\in X$, then $\chi^F_C(x,S)\subseteq x$.

To see that condition 3 is equivalent to condition 5, observe that if condition 5 holds, then for all $x\in X$, we have $\chi^F_C(s,t)\in [s]_X=x$ for all $s\in x$ and  $t\in S$. Thus $\chi^F_C(x,S)\subseteq x$. Conversely, if condition 3 holds, $\chi^F_C(s,t)\in \chi^F_C([s]_X,S)\subseteq [s]_X$ for all $s,t\in S$.

Condition 6 is clearly a trivial restatement of condition 5.

To see that conditions 6 and 7 are equivalent, observe that if condition 6 holds, and $s,t\in S$ satisfy $s\sim_{\bigvee_S(C)}t$, then $\chi^F_C(s,t)=t$, so $t=\chi^F_C(s,t)\sim_X s$. Thus $X\leq_S\bigvee_S(C)$. Conversely, if condition 7 holds, then since $\chi^F_C(s,t)\sim_{\bigvee_S(C)} s$ for all $s,t\in S$, we have $\chi^F_C(s,t)\sim_{X} s$.
\end{proof}
Here are some basic properties of $\vdash^F$.
\begin{proposition}\label{propgen}
Let $F=(S,B)$ be a finite factored set, let $C$ and $D$ be subsets of $B$, and let $X,Y\in \parts$ be partitions of $S$.
\begin{enumerate}
    \item If $X\leq_S Y$ and $C\vdash^F Y$, then $C\vdash^F X$.
    \item If $C\vdash^F X$ and $C\vdash^F Y$, then $C\vdash^F X\vee_S Y$.
    \item $B\vdash^F X$.
    \item $\{\}\vdash^F X$ if and only if $X=\text{Ind}_S$.
    \item If $C\subseteq D$ and $C\vdash^F X$, then $D\vdash^F X$.
    \item If $C\vdash^F X$ and $D\vdash^F X$, then $C\cap D\vdash^F X$.
\end{enumerate}
\end{proposition}
\begin{proof}
For the first 5 parts, we will use the equivalent definition from Proposition \ref{equivgen} that $C\vdash^F X$ if and only if $X\leq_S\bigvee_S(C)$.

Then 1 follows directly from the transitivity of $\leq_S$.

2 follows directly from the fact that any partition $Z$ satisfies $X\vee_S Y\leq Z$ if and only if $X\leq Z$ and $Y\leq Z$.

3 follows directly from the fact that $\bigvee_S(B)=\text{Dis}_S$ by Proposition \ref{templabel1}.

4 follows directly from the fact that $\bigvee_S(\{\})=\text{Ind}_S$, together with the fact that $X\leq_S \text{Ind}_S$ if and only if $X=\text{Ind}_S$.

5 follows directly from the fact that if $C\subseteq D$, then $\bigvee_S(C)\leq \bigvee_S(D)$.

Finally, we need to prove part 6. For this, we will use the equivalent definition from Proposition \ref{equivgen} that $C\vdash^F X$ if and only if $\chi^F_C(s,t)\sim_X s$ for all $s,t\in S$. Assume that for all $s,t\in S$, $\chi^F_C(s,t)\sim_X s$ and $\chi^F_D(s,t)\sim_X s$. Thus, for all $s,t\in S$, $\chi^F_{C\cap D}(s,t)=\chi^F_C(\chi^F_D(s,t),t)\sim_X\chi^F_D(s,t)\sim_X s$. Thus $C\cap D\vdash^F X$.
\end{proof}
Our main use of $\vdash^F$ will be in the definition of the history of a partition.
\subsection{History}
\begin{definition}[history of a partition]\label{templabel8}
Given a finite factored set $F=(S,B)$ and a partition $X\in \parts$, let $h^F(X)$ denote the smallest (according to the subset ordering) subset of $B$ such that $h^F(X)\vdash^F X$.
\end{definition}
The history of $X$, then, is the smallest set of factors $C \subseteq B$ such that if you're trying to figure out which part in $X$ any given $s \in S$ is in, it suffices to know what part $s$ is in within each of the factors in $C$. We can informally think of $h^F(X)$ as the smallest amount of information needed to compute $X$.
\begin{proposition}
Given a finite factored set $F=(S,B)$, and a partition $X\in\parts$, $h^F(X)$ is well-defined.
\end{proposition}
\begin{proof}
Fix a finite factored set $F=(S,B)$ and a partition $X\in\parts$, and let $h^F(X)$ be the intersection of all $C\subseteq B$ such that $C\vdash^F X$. It suffices to show that $h^F(X)\vdash^F X$; then $h^F(X)$ will clearly be the unique smallest (according to the subset ordering) subset of $B$ such that $h^F(X)\vdash^F X$.

Note that $h^F(X)$ is a finite intersection, since there are only finitely many subsets of $B$, and that $h^F(X)$ is an intersection of a nonempty collection of sets since $B\vdash^F X$. Thus, we can express $h^F(X)$ as a composition of finitely many binary intersections. By part 6 of Proposition \ref{propgen}, the intersection of two subsets that generate $X$ also generates $X$. Thus $h^F(X)\vdash^F X$. 
\end{proof}
Here are some basic properties of history.
\begin{proposition}\label{prophist}
Let $F=(S,B)$ be a finite factored set, and let $X,Y\in \parts$ be partitions of $S$.
\begin{enumerate}
    \item If $X\leq_S Y$, then $h^F(X)\subseteq h^F(Y)$.
    \item $h^F(X\vee_S Y)=h^F(X)\cup h^F(Y)$.
    \item $h^F(X)=\{\}$ if and only if $X=\text{Ind}_S$.
    \item If $S$ is nonempty, then $h^F(b)=\{b\}$ for all $b\in B$.
\end{enumerate}
\end{proposition}
\begin{proof}
The first 3 parts are trivial consequences of history's definition and Proposition \ref{propgen}.

For the fourth part, observe that $\{b\}\vdash^F b$ by condition 7 of Proposition \ref{equivgen}, $b$ is nontrivial, and since $S$ is nonempty $b$ is nonempty, so we have $\neg(\{\}\vdash^F b)$ by part 4 of Proposition \ref{propgen}. Thus $\{b\}$ is the smallest subset of $B$ that generates $b$.
\end{proof}
\subsection{Orthogonality}
We are now ready to define the notion of orthogonality between two partitions of $S$.
\begin{definition}[orthogonality]
Given a finite factored set $F=(S,B)$ and partitions $X,Y\in \parts$, we say $X$ is orthogonal to $Y$ (in $F$), written $\ortho{X}{Y}$, if $h^F(X)\cap h^F(Y)=\{\}$.

If $\neg(\ortho{X}{Y})$, we say $X$ is entangled with $Y$ (in $F$).
\end{definition}
We could also unpack this definition to not mention history or chimera functions. 
\begin{proposition}
Given a finite factored set $F=(S,B)$, and partitions $X,Y\in \parts$,  $\ortho{X}{Y}$ if and only if there exists a $C\subseteq B$ such that $X\leq_S\bigvee_S(C)$ and $Y\leq_S\bigvee_S(B\setminus C)$.
\end{proposition}
\begin{proof}
If there exists a $C\subseteq B$ such that $X\leq_S\bigvee_S(C)$ and $Y\leq_S\bigvee_S(B\setminus C)$, then $C\vdash^F X$ and $B\setminus C\vdash^F Y$. Thus, $h^F(X)\subseteq C$ and $h^F(Y)\subseteq B\setminus C$, so $h^F(X)\cap h^F(Y)=\{\}$.

Conversely, if $h^F(X)\cap h^F(Y)=\{\}$, let $C=h^F(X)$. Then $C\vdash^F X$, so $X\leq_S\bigvee_S(C)$, and $B\setminus C\supseteq h^F(Y)$, so $B\setminus C\vdash^F Y$, so $Y\leq_S\bigvee_S(B\setminus C)$.
\end{proof}
Here are some basic properties of orthogonality.
\begin{proposition}
Let $F=(S,B)$ be a finite factored set, and let $X,Y,Z\in \parts$ be partitions of $S$.
\begin{enumerate}
    \item If $\ortho{X}{Y}$, then $\ortho{Y}{X}$.
    \item If $\ortho{X}{Z}$ and $Y\leq_S X$, then $\ortho{Y}{Z}$.
    \item If $\ortho{X}{Z}$ and $\ortho{Y}{Z}$, then $\ortho{(X\vee_S Y)}{Z}$.
        \item $\ortho{X}{X}$ if and only if $X=\text{Ind}_S$.
\end{enumerate}
\end{proposition}
\begin{proof}
Part 1 is trivial from the symmetry in the definition. 

Parts 2, 3, and 4 follow directly from Proposition \ref{prophist}.
\end{proof}
\subsection{Time}
Finally, we can define our notion of time in a factored set.
\begin{definition}[(strictly) before]
Given a finite factored set $F=(S,B)$, and partitions $X,Y\in \parts$, we say $X$ is before $Y$ (in $F$), written $X\leq^F Y$, if $h^F(X)\subseteq h^F(Y)$.

We say $X$ is strictly before $Y$ (in $F$), written $X<^FY$, if $h^F(X)\subset h^F(Y)$.
\end{definition}
Again, we could also unpack this definition to not mention history or chimera functions. 
\begin{proposition}
Given a finite factored set $F=(S,B)$, and partitions $X,Y\in \parts$,  $X\leq^F Y$ if and only if every $C\subseteq B$ satisfying $Y\leq_S\bigvee_S(C)$ also satisfies $X\leq_S\bigvee_S(C)$.
\end{proposition}
\begin{proof}
Note that by part 7 of Proposition \ref{equivgen}, part 5 of Proposition \ref{propgen}, and the definition of history, $C$ satisfies $Y\leq_S\bigvee_S(C)$ if and only if $C\supseteq h^F(Y)$, and similarly for $X$. 

Clearly, if $h^F(Y)\supseteq h^F(X)$, every $C\supseteq h^F(Y)$ satisfies $C\supseteq h^F(X)$. Conversely, if $h^F(X)$ is not a subset of $h^F(Y)$, then we can take $C=h^F(Y)$, and observe that $C\supseteq h^F(Y)$ but not $C\supseteq h^F(X)$.
\end{proof}
Interestingly, we can also define time entirely as a closure property of orthogonality. We hold that the philosophical interpretation of time as a closure property on orthogonality is natural and transcends the ontology set up in this paper.
\begin{proposition}
Given a finite factored set $F=(S,B)$, and partitions $X,Y\in \parts$, $X\leq^F Y$ if and only if every $Z\in\parts$ satisfying $\ortho{Y}{Z}$ also satisfies $\ortho{X}{Z}$.
\end{proposition}
\begin{proof}
Clearly if $h^F(X)\subseteq h^F(Y)$, then every $Z$ satisfying $h^F(Y)\cap h^F(Z)=\{\}$ also satisfies $h^F(X)\cap h^F(Z)=\{\}$.

Conversely, if $h^F(X)$ is not a subset of $h^F(Y)$, let $b\in B$ be an element of $h^F(X)$ that is not in $h^F(Y)$. Assuming $S$ is nonempty, $b$ is nonempty, so we have $h^F(b)=\{b\}$, so $\ortho{Y}{b}$, but not $\ortho{X}{b}$. On the other hand, if $S$ is empty, then $X=Y=\{\}$, so clearly $X\leq^F Y$.
\end{proof}
Here are some basic properties of time.
\begin{proposition}
Let $F=(S,B)$ be a finite factored set, and let $X,Y,Z\in \parts$ be partitions of $S$.
\begin{enumerate}    
    \item $X\leq^F X$.
    \item If $X\leq^F Y$ and $Y\leq^F Z$, then $X\leq^F Z$.
    \item If $X\leq_S Y$, then $X\leq^F Y$.
    \item If $X\leq^F Z$ and $Y\leq^F Z$, then $(X\vee_S Y)\leq^F Z$.
\end{enumerate}
\end{proposition}
\begin{proof}
Part 1 is trivial from the definition.

Part 2 is trivial by transitivity of the subset relation.

Part 3 follows directly from part 1 of Proposition \ref{prophist}.

Part 4 follows directly from part 2 of Proposition \ref{prophist}.
\end{proof}
Finally, note that we can (circularly) redefine history in terms of time, thus partially justifying the names.
\begin{proposition}
Given a nonempty finite factored set $F=(S,B)$ and a partition $X\in \parts$, $h^F(X)=\{b\in B\mid b\leq^F X\}$.
\end{proposition}
\begin{proof}
Since $S$ is nonempty, part 4 of Proposition \ref{prophist} says that $h^F(b)=\{b\}$ for all $b\in B$. Thus $\{b\in B\mid b\leq^F X\}=\{b\in B\mid \{b\} \subseteq h^F(X)\}=\{b\in B\mid b \in h^F(X)\}=h^F(X)$.
\end{proof}

\section{Subpartitions and Conditional Orthogonality}

We now want to extend our notion of orthogonality to conditional orthogonality. This will take a bit of work. In particular, we will have to first extend our notions of partition \emph{generation} and \emph{history} to be defined on partitions of subsets of $S$.
\subsection{Generating a Subpartition}
\begin{definition}[subpartition]
A subpartition of a set $S$ is a partition of a subset of $S$. Let $\text{SubPart}(S)=\bigcup_{E\subseteq S}\parts[E]$ denote the set of all subpartitions of $S$.
\end{definition}
\begin{definition}[domain]
The domain of a subpartition $X$ of $S$, written $\text{dom}(X)$, is the unique $E\subseteq S$ such that $X\in\parts[E]$.
\end{definition}
\begin{definition}[restricted partitions]\label{defbare}
Given sets $S$ and $E$ and a partition $X$ of $S$, let $X|E$ denote the partition of $S\cap E$ given by $X|E=\{[e]_X\cap E\mid e\in E\}$.
\end{definition}
\begin{definition}[generating a subpartition]\label{genasubpart}
Given a finite factored set $F=(S,B)$, and $X\in \text{SubPart}(S)$, and a $C\subseteq B$, we say $C$ generates $X$ (in $F$), written $C\vdash^F X$, if $\chi^F_C(x,\text{dom}(X))=x$ for all $x\in X$.
\end{definition}
Note that this definition clearly coincides with Definition \ref{templabel9}, when $X$ has domain $S$. Despite the similarity of the definitions, the idea of generating a subpartition is a bit more complicated than the idea of generating a partition of $S$.

To see this, consider the following list of equivalent definitions. Notice that while the first five directly mirror their counterparts in Proposition \ref{equivgen}, the last two (and especially the last one) require an extra condition.
\begin{proposition}\label{equivsgen}
Let $F=(S,B)$ be a finite factored set, let $X\in \text{SubPart}(S)$ be a subpartition of $S$, let $E=\text{dom}(X)$ be the domain of $X$, and let $C$ be a subset of $B$. The following are equivalent.
\begin{enumerate}
    \item $C\vdash^F X$.
    \item $\chi^F_C(x,E)=x$ for all $x\in X$.
    \item $\chi^F_C(x,E)\subseteq x$ for all $x\in X$.
    \item $\chi^F_C(x,y)\subseteq x$ for all $x,y\in X$.
    \item $\chi^F_C(s,t)\in [s]_X$ for all $s,t\in E$.
    \item $\chi^F_C(s,t)\in E$ and $\chi^F_C(s,t)\sim_X s$ for all $s,t\in E$.
    \item $X\leq_E(\bigvee_S(C)|E)$ and $\chi^F_C(E,E)=E$.
\end{enumerate}
\end{proposition}
\begin{proof}
The equivalence of conditions 1 and 2 is by definition.

The equivalence of conditions 2 and 3 follows directly from the fact that $\chi^F_C(s,s)=s$ for all $s\in x$, so $\chi^F_C(x,E)\supseteq \chi^F_C(x,x)\supseteq x$.

To see that conditions 3 and 4 are equivalent, observe that since $E=\bigcup_{y\in X} y$, $\chi^F_C(x,E)=\bigcup_{y\in X}\chi^F_C(x,y)$. Thus, if $\chi^F_C(x,E)\subseteq x$, $\chi^F_C(x,y)\subseteq x$ for all $y\in X$, and conversely if $\chi^F_C(x,y)\subseteq x$ for all $y\in X$, then $\chi^F_C(x,E)\subseteq x$.

To see that condition 3 is equivalent to condition 5, observe that if condition 5 holds, then for all $x\in X$, we have $\chi^F_C(s,t)\in [s]_X=x$ for all $s\in x$ and  $t\in E$. Thus $\chi^F_C(x,E)\subseteq x$. Conversely, if condition 3 holds, $\chi^F_C(s,t)\in \chi^F_C([s]_X,E)\subseteq [s]_X$ for all $s,t\in E$.

Condition 6 is clearly a trivial restatement of condition 5.

To see that conditions 6 and 7 are equivalent, observe that if condition 6 holds, then $\chi_C^F(s,t)\in E$ for all $s,t\in E$, so $\chi_C^F(E,E)\subseteq E$, so $\chi_C^F(E,E)=E$. Further, if $s,t\in E$ satisfy $s\sim_{\bigvee_S(C)|E}t$, then $s\sim_{c}t$ for all $c\in C$, so $\chi^F_C(s,t)=t$, so $t=\chi^F_C(s,t)\sim_X s$. Thus $X\leq_E\bigvee_S(C)|E$. 

Conversely, if condition 7 holds, then for all $s,t\in E$, we have $\chi^F_C(s,t)\sim_{\bigvee_S(C)} s$, so $\chi^F_C(s,t)\sim_{\bigvee_S(C)|E} s$, and thus $\chi^F_C(s,t)\sim_{X} s$. Further, clearly $\chi_C^F(E,E)=E$ implies $\chi_C^F(s,t)\in E$ for all $s,t\in E$.

\end{proof}
The first half of condition 7 in the above proposition can be thought of as saying that the values of factors in $C$ are sufficient to distinguish between the parts of $X$.

The second half can be thought of as saying that no factors in $C$ become entangled with any factors outside of $C$ when conditioning on $E$. This second half is actually necessary (for example) to ensure that the set of all $C$ that generate $X$ is closed under intersection. As such, we will need this fact in order to extend our notion of history to arbitrary subpartitions.
\begin{proposition}\label{propsgen}
Let $F=(S,B)$ be a finite factored set, let $C$ and $D$ be subsets of $B$, let $X,Y,Z\in \text{SubPart}(S)$ be subpartitions of $S$, and let $\text{dom}(X)=\text{dom}(Y)=E$.
\begin{enumerate}
    \item If $X\leq_E Y$ and $C\vdash^F Y$, then $C\vdash^F X$.
    \item If $C\vdash^F X$ and $C\vdash^F Y$, then $C\vdash^F X\vee_E Y$.
    \item $B\vdash^F X$.
    \item $\{\}\vdash^F X$ if and only if $X=\text{Ind}_E$.
    \item If $C\vdash^F X$ and $D\vdash^F X$, then $C\cap D\vdash^F X$ and $C\cup D\vdash^F X$.
    \item If $X\subseteq Z$, and $C\vdash^F Z$, then $C\vdash^F X$.
\end{enumerate}
\end{proposition}
\begin{proof}
The first 4 parts will use the equivalent definition from Proposition \ref{equivsgen} that $C\vdash^F X$ if and only if $X\leq_S\bigvee_S(C)$. 1 and 2 are immediate from this definition.

3 follows directly from Definition \ref{genasubpart}. 

4 follows directly from the fact that $\bigvee_S(\{\})=\text{Ind}_S$, and $\text{Ind}_S|E=\text{Ind}_E$ so $X\leq_E\bigvee_S(C)|E$ if and only if $X=\text{Ind}_E$.

For part 5, we will use the equivalent definition from Proposition \ref{equivsgen} that $C\vdash^F X$ if and only if $\chi^F_C(s,t)\in [s]_X$ for all $s,t\in E$. Assume that for all $s,t\in E$, $\chi^F_C(s,t)\in [s]_X$ and $\chi^F_D(s,t)\in [s]_X$. Thus, for all $s,t\in E$, $\chi^F_{C\cap D}(s,t)=\chi^F_C(\chi^F_D(s,t),t)\in [\chi^F_D(s,t)]_X=[s]_X$. Similarly, for all $s,t\in E$, $\chi^F_{C\cup D}(s,t)=\chi^F_C(s,\chi^F_D(s,t))\in[s]_X$. Thus $C\cap D\vdash^F X$ and $C\cup D\vdash^F X$.

For part 6, we use the definition that $C\vdash^F X$ if and only if $\chi^F_C(x,y)\in x$ for all $x,y\in X$. Clearly if $X\subseteq Z$, and $\chi^F_C(x,y)\in x$ for all $x,y\in Z$, then $\chi^F_C(x,y)\in x$ for all $x,y\in X$.
\end{proof}
Note that while the set of $C$ that generate an $X\in \parts$ is closed under supersets, the set of $C$ that generate an $X\in \text{SubPart}(S)$ is merely closed under union. Further note that part 6 of Proposition \ref{propsgen} uses the subset relation on subpartitions, which is a slightly unnatural relation.
    
\subsection{History of a Subpartition}
\begin{definition}[history of a subpartition]Given a finite factored set $F=(S,B)$ and a subpartition $X\in\text{SubPart}(S)$, let $h^F(X)$ denote the smallest (according to the subset ordering) subset of $B$ such that $h^F(X)\vdash^F X$.
\end{definition}
\begin{proposition}
Given a finite factored set $F=(S,B)$, $h^F:\text{SubPart}(S)\rightarrow \mathcal{P}(B)$ is well-defined, and if $X$ is a partition of $S$, this definition coincides with Definition \ref{templabel8}.
\end{proposition}
\begin{proof}
Fix a finite factored set $F=(S,B)$ and a subpartition $X\in\text{SubPart}(S)$, and let $h^F(X)$ be the intersection of all $C\subseteq B$ such that $C\vdash^F X$. It suffices to show that $h^F(X)\vdash^F X$. Then $h^F(X)$ will clearly be the unique smallest (according to the subset ordering) subset of $B$ such that $h^F(X)\vdash^F X$. The fact that this definition coincides with Definition \ref{templabel8} if $X\in \parts$ is clear.

Note that $h^F(X)$ is a finite intersection, since there are only finitely many subsets of $B$, and that $h^F(X)$ is a nonempty intersection since $B\vdash^F X$. Thus, we can express $h^F(X)$ as a (possibly empty) composition of finitely many binary intersections. By part 5 of Proposition \ref{propsgen}, the intersection of two subsets that generate $X$ also generates $X$. Thus $h^F(X)\vdash^F X$. 
\end{proof}
We will now give five basic properties of the history of subpartitions, followed by two more properties that are less basic.
\begin{proposition}
Let $F=(S,B)$ be a finite factored set, let $X,Y,Z\in \text{SubPart}(S)$ be subpartitions of $S$, and let $\text{dom}(X)=\text{dom}(Y)=E$.
\begin{enumerate}
    \item If $X\leq_E Y$, then $h^F(X)\subseteq h^Y(Y)$.
    \item $h^F(X\vee_E Y)=h^F(X)\cup h^F(Y)$.
    \item If $X\subseteq Z$, then $h^F(X)\subseteq h^F(Z)$.
    \item $h^F(X)=\{\}$ if and only if $X=\text{Ind}_E$.
    \item If $S$ is nonempty, then $h^F(b)=\{b\}$ for all $b\in B$.
\end{enumerate}
\end{proposition}
\begin{proof}
Parts 1, 3, and 4 are trivial consequences of Proposition \ref{propsgen}, and part 5 is just a restatement of part 4 of Proposition \ref{prophist}.

For part 2, first observe that $h^F(X\vee_E Y)\supseteq h^F(X)\cup h^F(Y)$, by part 1 of Proposition \ref{propsgen}. Thus it suffices to show that $h^F(X)\cup h^F(Y)\supseteq h^F(X\vee_E Y)$, by showing that $h^F(X)\cup h^F(Y)\vdash^F X\vee_E Y$.

We will use condition 7 in Proposition \ref{equivsgen}. Clearly
\begin{equation} \begin{split} X & \leq_E({\bigvee}_E (h^F(X))|E) \\ & \leq_E({\bigvee}_S(h^F(X)\cup h^F(Y))|E),\end{split} \end{equation}
and similarly,
\begin{equation} \begin{split} Y & \leq_E({\bigvee}_E(h^F(Y))|E) \\ & \leq_E({\bigvee}_S(h^F(X)\cup h^F(Y))|E).\end{split} \end{equation}
Thus, $X\vee_E Y \leq_E (\bigvee_S(h^F(X)\cup h^F(Y))|E)$. 

Next, we need to show that $\chi_{h^F(X)\cup h^F(Y)}^F(E,E)=E$. Clearly $E\subseteq \chi_{h^F(X)\cup h^F(Y)}^F(E,E)$. 

Let $s$ and $t$ be elements of $E$, and observe that  $\chi_{h^F(X)\cup h^F(Y)}^F(s,t)=\chi_{h^F(X)}^F(s,\chi_{h^F(Y)}^F(s,t))$. We have that $\chi_{h^F(Y)}^F(s,t)\in E$, since $\chi_{h^F(Y)}^F(E,E)= E$. Thus, we also have that $\chi_{h^F(X)}^F(s,\chi_{h^F(Y)}^F(s,t))\in E$, since $\chi_{h^F(X)}^F(E,E)= E$. Thus, $\chi_{h^F(X)\cup h^F(Y)}^F(E,E)\subseteq E$.

Thus we have that $X\vee_E Y \leq_E (\bigvee_S(h^F(X)\cup h^F(Y))|E)$ and $\chi_{h^F(X)\cup h^F(Y)}^F(E,E)=E$. Thus, by condition 7 in Proposition \ref{equivsgen}, $h^F(X)\cup h^F(Y)\vdash^F X\vee_E Y$, so $h^F(X\vee_E Y)=h^F(X)\cup h^F(Y)$.
\end{proof}
\begin{lemma2}\label{histlemma1}
Let $F=(S,B)$ be a finite factored set, and let $X,Y\in \parts[E]$ be subpartitions of $S$ with the same domain. If $h^F(X)\cap h^F(Y)=\{\}$, then $h^F(X)=h^F(X|y)$ for all $y\in Y$.
\end{lemma2}
\begin{proof}
Let $F=(S,B)$ be a finite factored set, let $E\subseteq S$, and let $X,Y\in \parts[E]$.

We start by showing that $(B\setminus h^F(X))\vdash^F Y$ and $(B\setminus h^F(Y))\vdash^F X$. Observe that $\chi_{B\setminus h^F(X)}(E,E)=\chi_{ h^F(X)}(E,E)=E$. Further observe that $B\setminus h^F(X)\supseteq h^F(Y)$, so $\bigvee_S(B\setminus h^F(X))\geq_S \bigvee_S(h^F(Y))$, so $(\bigvee_S(B\setminus h^F(X))|E)\geq_E (\bigvee_S(h^F(Y))|E)\geq_E Y$. Thus, $(B\setminus h^F(X))\vdash^F Y$. Symmetrically, $(B\setminus h^F(Y))\vdash^F X$.

Fix some $y\in Y$. We start by showing that $h^F(X)\supseteq h^F(X|y)$. 

We have that $\chi^F_{B\setminus h^F(X)}(y, E)\subseteq y$, so $\chi^F_{h^F(X)}(E, y)\subseteq y$, so for all $x\in X$, we have $\chi^F_{h^F(X)}(x\cap y, y)\subseteq y$. We also have $\chi^F_{h^F(X)}(x\cap y, y)\subseteq \chi^F_{h^F(X)}(x, E)\subseteq x$. Thus  $\chi^F_{h^F(X)}(x\cap y, y)\subseteq x\cap y$. Every element of $X|y$ is of the form $x\cap y$ for some $x\in X$, so we have $h^F(X)\vdash^F(X|y)$, so $h^F(X)\supseteq h^F(X|y)$.

Next, we need to show that $h^F(X)\subseteq h^F(X|y)$. For this, it suffices to show that $h^F(X|y)\vdash^F X$. Let $s,t$ be arbitrary elements of $E$. It suffices to show that $\chi^F_{h^F(X|y)}(s,t)\in [s]_X$. 

First, observe that since $(B\setminus h^F(Y))\supseteq h^F(X)\supseteq h^F(X|y)$, we have that $\chi^F_{h^F(X|y)}(s,t)=\chi^F_{B\setminus h^F(Y)}(\chi^F_{h^F(X|y)}(s,t),t)$.

Let $r$ be an arbitrary element of $y$. We thus have:

\begin{equation}
\begin{split}
\chi^F_{h^F(X|y)}(s,t) & = \chi^F_{B\setminus h^F(Y)}(\chi^F_{h^F(X|y)}(s,t),t) \\
 & = \chi^F_{B\setminus h^F(Y)}(\chi^F_{h^F(Y)}(r,\chi^F_{h^F(X|y)}(s,t)),t)\\
  & = \chi^F_{B\setminus h^F(Y)}(\chi^F_{h^F(X|y)}(\chi^F_{h^F(Y)}(r,s),\chi^F_{h^F(Y)}(r,t)),t).
\end{split}
\end{equation}
Let $s^\prime=\chi^F_{h^F(X|y)}(\chi^F_{h^F(Y)}(r,s),\chi^F_{h^F(Y)}(r,t))$. Note that $\chi^F_{h^F(Y)}(r,t)$ and $\chi^F_{h^F(Y)}(r,s)$ are both in $y$. Thus we have that $s^\prime\in [\chi^F_{h^F(Y)}(r,s)]_{(X|y)}$. Since $(B\setminus h^F(Y))\vdash^F X$, $\chi^F_{h^F(Y)}(r,s)=\chi^F_{B\setminus h^F(Y)}(s,r)\in [s]_X$. Thus $[\chi^F_{h^F(Y)}(r,s)]_{(X|y)}\subseteq[\chi^F_{h^F(Y)}(r,s)]_{X}=[s]_X$, so $s^\prime\in[s]_X$.

We have that $\chi^F_{h^F(X|y)}(s,t) =\chi^F_{B\setminus h^F(Y)}(s^\prime,t)$. However, since $B\setminus h^F(Y)\vdash^F X$, we have $\chi^F_{B\setminus h^F(Y)}(s^\prime,t)\in [s^\prime]_X=[s]_X$. Thus, $h^F(X)\subseteq h^F(X|y)$, so $h^F(X)=h^F(X|y)$.

\end{proof}
\begin{lemma2}\label{histlemma2}
Let $F=(S,B)$ be a finite factored set. Let $E \subseteq S$ and let $X,Y\in \parts[E]$ be subpartitions of $S$ with the same domain. Then $h^F(X\vee_EY)=h^F(X)\cup \bigcup_{x\in X}h^F(Y|x)$.
\end{lemma2}
\begin{proof}
Since $X\leq_E X\vee_E Y$, we have $h^F(X)\subseteq h^F(X\vee_E Y)$. Similarly, for all $x\in X$, since $Y|x\subseteq X\vee_E Y$, we have $h^F(Y|x)\subseteq h^F(X\vee_E Y)$. Thus, $h^F(X\vee_E Y)\supseteq h^F(X)\cup \bigcup_{x\in X}h^F(Y|x).$ We still need to show that $h^F(X\vee_E Y)\subseteq h^F(X)\cup \bigcup_{x\in X}h^F(Y|x).$

We start with the special case where $|X|=2$. Let $X=\{x_0,x_1\}$. In this case, we want to show that $h^F(X\vee_EY)=h^F(X)\cup h^F(Y|x_0)\cup h^F(Y|x_0)$. Let $C=h^F(X)$, let $C_0=h^F(Y|x_0)$, and let $C_1=h^F(Y|x_1)$.

Consider arbitrary $s,t\in E$. Without loss of generality, assume that $s\in x_0$, and let $y=[s]_Y$. It suffices to show that $\chi^F_{C\cup C_0\cup C_1}(s,t)\in x_0\cap y$. Fix some $r\in x_1$.
\begin{equation}
\begin{split}
\chi^F_{C\cup C_0\cup C_1}(s,t) & = \chi^F_{C_0}(s,\chi^F_{C}(s,\chi^F_{C_1}(s,t))) \\
 & = \chi^F_{C_0}(s,\chi^F_{C}(s,\chi^F_{C}(r,\chi^F_{C_1}(s,t))))\\
  & = \chi^F_{C_0}(s,\chi^F_{C}(s,\chi^F_{C_1}(\chi^F_{C}(r,s),\chi^F_{C}(r,t)))).
\end{split}
\end{equation}
Observe that $\chi^F_{C}(r,s)$ and $\chi^F_{C}(r,t)$ are both in $x_1$, so $\chi^F_{C_1}(\chi^F_{C}(r,s),\chi^F_{C}(r,t))\in x_1,$ and thus is in $E$. Combining this with the fact that $s\in x_0$ gives us that $\chi^F_C(s,\chi^F_{C_1}(\chi^F_{C}(r,s),\chi^F_{C}(r,t)))\in x_0.$ Thus, since $s\in x_0\cap y$, $\chi^F_{C\cup C_0\cup C_1}(s,t)=\chi^F_{C_0}(s,\chi^F_{C}(s,\chi^F_{C_1}(\chi^F_{C}(r,s),\chi^F_{C}(r,t))))\in x_0\cap y$.

Now, consider the case where $|X|\neq 2$. If $|X|=0$, then $E=\{\}$, so all subpartitions involved are empty, and thus have the same (empty) history. If $|X|=1$, let $X=\{E\}$. Then
\begin{equation} \begin{split} h^F(X\vee_E Y) & =h^F(Y) \\ & =h^F(Y|E)\subseteq h^F(X)\cup h^F(Y|E) \\ & =h^F(X)\cup \bigcup_{x\in X}h^F(Y|x). \end{split} \end{equation}
Thus, we can restrict our attention to the case where $|X|\geq 3$. 

Observe that $X\vee_E Y=\bigvee_E(\{(Y|x)\cup\{E\setminus x\}\mid x\in X\})$. Thus $h^F(X\vee_E Y)=\bigcup_{x\in X}h^F((Y|x)\cup\{E\setminus x\})$. However, from the case where $|X|=2$, we have 
\begin{equation}
\begin{split}
h^F((Y|x)\cup\{E\setminus x\}) & =h^F(\{x,E\setminus x\}\vee_E((Y|x)\cup\{E\setminus x\})) \\
 & =h^F(\{x,E\setminus x\})\cup h^F(\{E\setminus x\})\cup h^F(Y|x).\\
\end{split}
\end{equation}
 $h^F(\{E\setminus x\})$ is empty, so this gives us that $h^F(X\vee_E Y)=\bigcup_{x\in X}(h^F(Y|x)\cup h^F(\{x,E\setminus x\}))$. Since $\bigvee_E(\{\{x,E\setminus x\}\mid x\in X\})=X$, $\bigcup_{x\in X} h^F(\{x,E\setminus x\})=h^F(X)$, so we have $h^F(X\vee_EY)=h^F(X)\cup \bigcup_{x\in X}h^F(Y|x)$.
\end{proof}

\subsection{Conditional Orthogonality}
We can also extend our notions of orthogonality and time to subpartitions.
\begin{definition}
Let $F=(S,B)$ be a finite factored set. Let $X,Y\in \text{SubPart}(S)$ be subpartitions of $S$. We write $X\perp^F Y$ if $h^F(X)\cap h^F(Y)=\{\}$, we write $X\leq^F Y$ if $h^F(X)\subseteq h^F(Y)$, and we write $X<^F Y$ if $h^F(X)\subset h^F(Y)$.
\end{definition}
We give this definition in general, but it is not clear whether orthogonality and time should be considered philosophically meaningful when the domains of the inputs differ from each other. Further, the temporal structure of subpartitions will mostly be outside the scope of this paper, and the orthogonality structure on subpartitions will mostly just be used for the following pair of definitions.

\begin{definition}[conditional orthogonality given a subset]
Given a finite factored set $F=(S,B)$, partitions $X,Y\in \parts$, and $E\subseteq S$, we say $X$ and $Y$ are orthogonal given $E$ (in $F$), written $\co{X}{Y}{E}$, if $\ortho{(X|E)}{(Y|E)}$.
\end{definition}
\begin{definition}[conditional orthogonality]
Given a finite factored set $F=(S,B)$, and partitions $X,Y, Z\in \parts$, if $\co{X}{Y}{z}$ for all $z\in Z$, then we say $X$ and $Y$ are orthogonal given $Z$ (in $F$), written $\co{X}{Y}{Z}$.
\end{definition}
Unconditioned orthogonality can be thought of as a special case of conditional orthogonality, where you condition on the indiscrete partition.
\begin{proposition}
Given a finite factored set $F=(S,B)$ and partitions $X,Y\in \parts$, $\ortho{X}{Y}$ if and only if $\co{X}{Y}{\text{Ind}_S}$.
\end{proposition}
\begin{proof}
If $S=\{\}$, then there is only one partition $X=\{\}$, and $\ortho{X}{X}$ holds. Also, since $\text{Ind}_S$ is empty, $\co{X}{X}{\text{Ind}_S}$ holds vacuously.

If $S\neq\{\}$, then $\text{Ind}_S=\{S\}$, so $\co{X}{Y}{\text{Ind}_S}$ if and only if $\co{X}{Y}{S}$ if and only if $\ortho{X|S}{Y|S}$ if and only if $\ortho{X}{Y}$.
\end{proof}
The primary combinatorial structure of finite factored sets that we will be interested in is the structure of orthogonality ($\ortho{X}{Y}$), conditional orthogonality ($\co{X}{Y}{Z}$), and time ($X\leq^F Y$ and $X<^F Y$)
on inputs that are partitions.

We now will show that conditional orthogonality satisfies (a slight modification of) the axioms for a compositional semigraphoid. 
\begin{theorem}\label{semigraphoid}
Let $F=(S,B)$ be a finite factored set, and let $X,Y,Z,W\in \parts$ be partitions of $S$.
\begin{enumerate}
    \item If $\co{X}{Y}{Z}$, then $\co{Y}{X}{Z}$. (symmetry)
    \item If $\co{X}{(Y\vee_S W)}{Z}$, then $\co{X}{Y}{Z}$ and $\co{X}{W}{Z}$. (decomposition)
    \item If $\co{X}{(Y\vee_S W)}{Z}$, then $\co{X}{Y}{(Z\vee_S W)}$. (weak union)
    \item If $\co{X}{Y}{Z}$ and $\co{X}{W}{(Z\vee_S Y)}$, then $\co{X}{(Y\vee_S W)}{Z}$. (contraction)
    \item If $\co{X}{Y}{Z}$ and $\co{X}{W}{Z}$, then $\co{X}{(Y\vee_S W)}{Z}$. (composition)
\end{enumerate}
\end{theorem}
\begin{proof}
Symmetry is clear from the definition.

Decomposition and composition both follow directly from the fact that for all $z\in Z$, $h^F((Y\vee_S W)|z)=h^F((Y|z)\vee_z (W|z))=h^F(Y|z)\cup h^F(W|z)$.

For weak union, assume that $\co{X}{(Y\vee_S W)}{Z}$. Thus, for all $z\in Z$, $h^F(X|z)\cap h^F((Y\vee_S W)|z)=\{\}$.

In particular, this means that $h^F(X|z)\cap h^F(W|z)=\{\}$, so by Lemma \ref{histlemma1}, for all $w\in W$, $h^F(X|z)=h^F(X|w\cap z)$.

Further, we have that for all $w\in W$, $h^F(Y|w\cap z)\subseteq h^F(Y \vee_S W|z)$. Thus, for all $w\in W$, $h^F(X|w\cap z)\cap h^F(Y|w\cap z)=\{\}$, which since every element of $W\vee_S Z$ is of the form $w\cap z$ for some $w\in W$ and $z\in Z$, means that $\co{X}{Y}{(Z\vee_S W)}$.

Finally, for contraction, assume that $\co{X}{Y}{Z}$ and $\co{X}{W}{Z\vee_S Y}$.

Fix some $z\in Z$. We want to show that $h^F(X|z)\cap h^F((Y\vee_S W)|z)=\{\}$. We have that $h^F((Y\vee_S W)|z)=h^F((Y|z)\vee_z (W|z))$, and by Lemma \ref{histlemma2}, $h^F((Y|z)\vee_z (W|z))=h^F(Y|z)\cup\bigcup_{y\in Y}h^F(W|(y\cap z))$. Thus, it suffices to show that $h^F(X|z)\cap h^F(Y|z)=\{\}$ and $h^F(X|z)\cap h^F(W|(y\cap z))=\{\}$ for all $y\in Y$.

The fact that $h^F(X|z)\cap h^F(Y|z)=\{\}$ follows directly from $\co{X}{Y}{Z}$. 

Fix a $y\in Y$. If $y\cap z=\{\}$, then $h^F(W|(y\cap z))=\{\}$, so $h^F(X|z)\cap h^F(W|(y\cap z))=\{\}$.

Otherwise, we have $h^F(X|z)=h^F(X|(y\cap z))$ by Lemma \ref{histlemma1}, and we have that $h^F(X|(y\cap z))\cap h^F(W|(y\cap z))=\{\}$, since $\co{X}{W}{Z\vee_S Y}$, so we have $h^F(X|z)\cap h^F(W|(y\cap z))=\{\}$. 

Thus, $\co{X}{(Y\vee_S W)}{Z}$.
\end{proof}

The first four parts of Theorem \ref{semigraphoid} are essentially the semigraphoid axioms. The difference is that the semigraphoid axioms are normally defined as a ternary relation on disjoint sets of variables. We use partitions instead of sets of variables, use common refinement instead of union, and have no need for the disjointness condition. The fifth part (composition) is a converse to the decomposition axiom that is sometimes added to define a compositional semigraphoid.

The results in this paper will not depend on the theory of compositional semigraphoids, so we will not need to make the analogy any more explicit, but it is nice to note the similarity to existing well-studied structures.

We also get a nice relationship between conditional orthogonality and the refinement order.

\begin{proposition}
Let $F=(S,B)$ be a finite factored set, and let $X,Y\in \parts$ be partitions of $S$. $\co{X}{X}{Y}$ if and only if $X\leq_S Y$.
\end{proposition}
\begin{proof}
If $\co{X}{X}{Y}$, then for all $y\in Y$, $h^F(X|y)=\{\}$, so $X|y=\text{ind}_y$, so for all $s,t\in y$, we have $s\sim_{X|y} t$, and thus $s\sim_{X} t$. Thus, for all $s,t\in S$, if $s\sim_Y t$, then $s\sim_X t$. Thus $X\leq_S Y$.

Conversely, if $X\leq_S Y$, observe that for all $y\in Y$, $X|y=\text{ind}_y$, so $h^F(X|y)=\{\}$. Thus, $\co{X}{X}{Y}$.
\end{proof}
\section{Polynomials and Probability}
In this section, given a finite factored set $F=(S,B)$, we will show how to associate each $E\subseteq S$ with a characteristic polynomial, $Q^F_E$. We will discuss how to factor these characteristic polynomials, and use these characteristic polynomials to build up to the fundamental theorem of finite factored sets, which associates conditional orthogonality with conditional independence in probability distributions.
\subsection{Characteristic Polynomials}
\begin{definition} 
Given a finite factored set $F=(S,B)$, let $\text{Poly}^F$ denote the ring of polynomials with coefficients in $\mathbb{R}$ and variables in $\mathcal{P}(S)$. 
\end{definition}
\begin{definition} 
Given a finite factored set $F=(S,B)$, a $p\in\text{Poly}^F$, and an $f:\mathcal{P}(S)\rightarrow \mathbb R$, we write $p(f)\in \mathbb{R}$ for the evaluation of $p$ at $f$, computed by replacing each $E\subseteq S$ with $f(E)$.
\end{definition}
\begin{definition} 
Given a finite factored set $F=(S,B)$ and a polynomial $p\in \text{Poly}^F$, $\text{supp}(p)\subseteq \mathcal{P}(S)$ denotes the set of all variables $v\in\mathcal{P}(S)$ that appear in $p$. $\text{supp}(p)$ is called the support of $p$.
\end{definition}
\begin{definition} 
Given a finite factored set $F=(S,B)$, and an $E\subseteq S$, let $Q^F_E\in \text{Poly}^F$ be given by $Q^F_E=\sum_{s\in E}\prod_{b\in B}[s]_b$. $Q^F_E$ is called the characteristic polynomial of $E$ (in $F$).
\end{definition}

We will be building up to an understanding of how to factor $Q^F_E$ into irreducibles. For that, we will first need to give some basic notation for manipulating polynomials in $\text{Poly}^F$.

\begin{definition}
Given a finite factored set $F=(S,B)$, an $s\in S$, and a $C\subseteq B$, let $\text{mono}^F_C(s)\in \text{Poly}^F$ be given by $\text{mono}^F_C(s)=\prod_{b\in C}[s]_b$. 
\end{definition}

\begin{definition}
Given a finite factored set $F=(S,B)$, an $E\subseteq S$, and a $C\subseteq B$, let $\text{monos}^F_C(E)\in \mathcal{P}(\text{Poly}^F)$ be given by $\text{monos}^F_C(E)=\{\text{mono}^F_C(s)\mid s\in E\}$. 
\end{definition}

\begin{definition}
Given a finite factored set $F=(S,B)$, an $E\subseteq S$, and a $C\subseteq B$, let $\text{poly}^F_C(E)\in \text{Poly}^F$ be given by $\text{poly}^F_C(E)=\sum_{m\in \text{monos}^F_C(E)}m$. 
\end{definition}

\begin{proposition}
Let $F=(S,B)$ be a finite factored set, and let $E\subseteq S$. Then $Q^F_E=\text{poly}^F_B(E)$.
\end{proposition}
\begin{proof}
We start by showing that for all $s\neq t\in S$, $\text{mono}^F_B(s)\neq \text{mono}^F_B(t)$. 

Let $s\neq t\in S$ be arbitrary. By Proposition \ref{templabel1}, if $s\neq t$, there must be some $b\in B$ such that $[s]_b\neq [t]_b$. Then, note that $[s]_b\in \text{supp}(\text{mono}^F_B(s))$. If $[s]_b$ were also in $\text{supp}(\text{mono}^F_B(t))$, then $t$ would be in both $[s]_b$ and $[t]_b$, contradicting the fact that these two sets are disjoint. Therefore $\text{mono}^F_B(s)\neq \text{mono}^F_B(t)$.

Thus $\text{monos}^F_B(E)$ has exactly one element for each element of $E$, so we have that $\sum_{m\in \text{monos}^F_B(E)}m=\sum_{s\in E}\text{mono}^F_B(s)=Q_E^F$.
\end{proof}
\begin{proposition}\label{factor1}
Let $F=(S,B)$ be a finite factored set, and let $E_0,E_1\subseteq S$ be subsets of $S$. Let $C_0,C_1\subseteq B$ be disjoint subsets of $B$. Let $E_2=\chi^F_{C_0}(E_0,E_1)$, and let $C_2=C_0\cup C_1$. Then $\text{poly}^F_{C_2}(E_2)=\text{poly}^F_{C_0}(E_0)\cdot \text{poly}^F_{C_1}(E_1)$.
\end{proposition}
\begin{proof}
For $i\in \{0,1,2\}$, let $M_i=\text{monos}^F_{C_i}(E_i)$. We will start by showing that $f:M_0\times M_1\rightarrow M_2$, given by $f(m_0,m_1)=m_0m_1$, is a well-defined function and a bijection.

First, observe that it follows immediately from the definition that for all $s_0,s_1\in S$, if $s_2=\chi^F_{C_0}(s_0,s_1)$ we have that $\text{mono}^F_{C_0}(s_0)=\text{mono}^F_{C_0}(s_2)$, $\text{mono}^F_{C_1}(s_1)=\text{mono}^F_{C_1}(s_2)$, and $\text{mono}^F_{C_0}(s_2)\cdot\text{mono}^F_{C_1}(s_2)=\text{mono}^F_{C_2}(s_2)$. Combining these, we get that $\text{mono}^F_{C_0}(s_0)\cdot\text{mono}^F_{C_1}(s_1)=\text{mono}^F_{C_2}(\chi^F_{C_0}(s_0,s_1))$.

For all $(m_0,m_1)\in M_0\times M_1$, there exists some $s_0\in E_0$ such that $m_0=\text{mono}^F_{C_0}(s_0)$, and some $s_1\in E_1$ such that $m_1=\text{mono}^F_{C_1}(s_1)$, and this gives us that $m_0m_1=\text{mono}^F_{C_0}(s_0)\text{mono}^F_{C_1}(s_1)=\text{mono}^F_{C_2}(\chi^F_C(s_0,s_1))\in M_2$. Thus, $f$ is well-defined.

To see that $f$ is surjective, observe that for all $m_2\in M_2$, there exists an $s_2\in E_2$ such that $m_2=\text{mono}^F_{C_2}(s_2)$, and there exist $s_0\in E_0$ and $s_1\in E_1$ such that $s_2=\chi^F_C(s_0,s_1)$, and we have $f(\text{mono}^F_{C_0}(s_0),\text{mono}^F_{C_1}(s_1))=m_2$.

To see that $f$ is injective, observe that for $i\in\{0,1\}$, for all $m_i\in M_i$, $\text{supp}(m_i)\subseteq \bigcup_{b\in C_i} b$. Further, $\bigcup_{b\in C_0} b$ and $\bigcup_{b\in C_1} b$ are disjoint. Thus, for all $m_0\in M_0$ and $m_1\in M_1$, $\text{supp}(m_i)=\text{supp}(m_0m_1)\cap \bigcup_{b\in C_i} b$.

This means that for all $m_0,m_0^\prime\in M_0$ and $m_1,m_1^\prime\in M_1$, if $m_0m_1=m_0^\prime m_1^\prime$, then $\text{supp}(m_0)=\text{supp}(m_0^\prime)$ and $\text{supp}(m_1)=\text{supp}(m_1^\prime)$. However, every monomial in $M_0$ or $M_1$ is just equal to the product of all variables in its support. Thus $m_0=\prod_{v\in \text{supp}(m_0)}v=m_0^\prime$ and $m_1=\prod_{v\in \text{supp}(m_1)}v=m_1^\prime$. Thus $f$ is injective, and thus a bijection between $M_0\times M_1$ and $M_2$.

Now, we have that
\begin{equation}
\begin{split}
\text{poly}^F_{C_0}(E_0)\cdot \text{poly}^F_{C_1}(E_1) & = \left(\sum_{m_0\in M_0}m_0\right)\left(\sum_{m_1\in M_1}m_1\right) \\
 & = \sum_{m_0\in M_0}\sum_{m_1\in M_1}m_0m_1\\
 & = \sum_{(m_0,m_1)\in M_0\times M_1}m_0m_1\\
  & = \sum_{(m_0,m_1)\in M_0\times M_1}f(m_0,m_1)\\
    & = \sum_{m_2\in M_2}m_2\\
      & = \text{poly}^F_{C_2}(E_2).
\end{split}
\end{equation}
\end{proof}

\begin{proposition}\label{factor2}
Let $F=(S,B)$ be a finite factored set, and let $E$ be a nonempty subset of $S$. If $p$ divides $Q^F_E$, then $p=r\cdot\text{poly}^F_C(E)$, for some $r\in \mathbb{R}$ and $C\subseteq B$.
\end{proposition}
\begin{proof}
Let $F=(S,B)$ be a finite factored set, and let $E$ be a nonempty subset of $S$. Let $p,q\in \text{Poly}^F$ satisfy  $pq=Q^F_E$. We thus
must have $\text{supp}(p)\cup\text{supp}(q)=\text{supp}(Q^F_E)$.

If there were some $T\in \text{supp}(p)\cap\text{supp}(q)$, then the degree of $T$ in $Q^F_E$ would be at least 2, contradicting the definition of $Q^F_E$ and Corollary \ref{basisdisjoint}. Thus, $\text{supp}(p)\cap\text{supp}(q)=\{\}$.

There can be no combining like terms, then, in the product $pq$. The monomial terms in $Q^F_E$ are in bijective correspondence to the pairs of monomial terms in $p$ and monomial terms in $q$.

In particular, this means that since all the coefficients in $pq$ are equal to 1, all the coefficients in $p$ must be equal to some $r\in \mathbb{R}$,
and all of the coefficients in $q$ must be equal to $1/r$.

Further, for all $b\in B$, if $b\cap \text{supp}(p)$ is nonempty, $b\cap \text{supp}(q)$ must be empty, since otherwise $Q^F_E$ would contain a term with two factors in $b$, which clearly never happens according to the definition of $Q^F_E$. 

Since $E$ is nonempty, for each $b\in B$ there must be some $T\in b\cap \text{supp}(Q^F_E)$. Thus at least one of $b\cap \text{supp}(p)$ and $b\cap \text{supp}(q)$ must be nonempty, so exactly one of $b\cap \text{supp}(p)$ and $b\cap \text{supp}(q)$ must be nonempty.

Let $C$ be the set of all $b\in B$ such that $b\cap \text{supp}(p)$ is nonempty. 

For every $b\in C$, every term of $Q^F_E$ has exactly one factor in $b$. Thus, every term in $p$ has exactly one factor in $b$. These cover all variables in the support of $p$, so each term in $p$ must have total degree $|C|$.

For each $m\in \text{monos}^F_C(E)$, $m$ divides a term in $Q^F_E$.

Since $m$ has no common support with $q$, $m$ must also divide a term in $p$. Thus $r\cdot m$ must be a term in $p$. Conversely, every term in $p$ divides a term in $Q^F_E$, and thus must be in $\text{monos}^F_C(E)$. Thus every term in $p$ is of the form $r\cdot m$ for some $m\in\text{monos}^F_C(E)$. Thus $p=\sum_{m\in\text{monos}^F_C(E)}r\cdot m= r\cdot\text{poly}^F_C(E)$.
\end{proof}
\subsection{Factoring Characteristic Polynomials}
We will now show how to factor characteristic polynomials into irreducibles. 
\begin{definition}
Given a finite factored set $F=(S,B)$, and a nonempty subset $E\subseteq S$, let $\text{Irr}^F(E)\subseteq\mathcal{P}(B)$ denote the set of all $C\subseteq B$ such that:
\begin{enumerate}
    \item $C$ is nonempty,
    \item $\chi^F_C(E,E)=E$, and
    \item there is no nonempty strict subset $D\subset C$ such that $\chi^F_D(E,E)=E$.
\end{enumerate}
\end{definition}

\begin{proposition}
Let $F=(S,B)$ be a finite factored set, and let $E$ be a nonempty subset of $S$. Then $\text{Irr}^F(E)\in\parts[B]$.
\end{proposition}
\begin{proof}
Let $F=(S,B)$ be a finite factored set, and let $E$ be a nonempty subset of $S$. It suffices to show that the sets in $\text{Irr}^F(E)$ are pairwise disjoint and cover $B$. 

We start by showing that the set of all $C\subseteq B$ satisfying $\chi^F_C(E,E)=E$ is closed under intersection. Indeed, if $\chi^F_{C_0}(E,E)=E$ and $\chi^F_{C_1}(E,E)=E$, then $\chi^F_{C_0\cap C_1}(E,E)=\chi^F_{C_0}(E,\chi^F_{C_1}(E,E))=\chi^F_{C_0}(E,E)=E$.

Next, observe that $\chi^F_B(E,E)=E$. Thus, for all $b\in B$, we can consider $C_b=\bigcap_{C\subseteq B,b\in C,\chi^F_C(E,E)=E}C$. Since $C_b$ is an intersection of a finite nonempty collection of sets $C$ satisfying  $\chi^F_C(E,E)=E$, we have that $\chi^F_{C_b}(E,E)=E$. Further, $b\in C_b$, so $C_b$ is nonempty.

Assume for the purpose of contradiction that there is some nonempty strict subset $D\subset C_b$ such that $\chi^F_{D}(E,E)=E$. If $b\in D$, then we have a contradiction by the definition of $C_b$. If $b\notin D$, then note that $\chi^F_{B\setminus D}(E,E)=E$, so  $\chi^F_{C_b\setminus D}(E,E)=E$, and $C_b\setminus D$ is a nonempty strict subset of $C_b$ that contains $b$, contradicting the definition of $C_b$.

Thus $C_b\in \text{Irr}^F(E)$ for all $b\in B$, and since $b\in C_b$, this means that the sets in $\text{Irr}^F(E)$ cover $B$.

Next, we need to show that the sets in $\text{Irr}^F(E)$ are pairwise disjoint. Let $C_0,C_1\in \text{Irr}^F(E)$ be arbitrary distinct elements. We have that $\chi^F_{C_0\cap C_1}(E,E)=E$, and $C_0\cap C_1$ is a subset of $C_0$ and $C_1$, and thus a strict subset of at least one of them. Thus $C_0\cap C_1$ is empty.

Thus $\text{Irr}^F(E)\in\parts[B]$.
\end{proof}
The following two propositions constitute a factorization of $Q^F_E$ into irreducibles.

\begin{proposition}
Let $F=(S,B)$ be a finite factored set, and let $E$ be a nonempty subset of $S$. Then $Q^F_E=\prod_{C\in\text{Irr}^F(E)}\text{poly}^F_C(E)$.
\end{proposition}
\begin{proof}
Let $F=(S,B)$ be a finite factored set, and let $E$ be a nonempty subset of $S$. Let $n=|\text{Irr}^F(E)|$, and let $\text{Irr}^F(E)=\{C_0,\dots, C_{n-1}\}$. For $0\leq k<n$, let $C_{\leq k}=\bigcup_{i=0}^{k} C_i$.

We will show by induction on $k$ that $\prod_{i=0}^{k}\text{poly}^F_{C_i}(E)=\text{poly}^F_{C_{\leq k}}(E)$ for all $0\leq k<n$.

If $k=0$, the result is trivial, as $\prod_{i=0}^{0}\text{poly}^F_{C_i}(E)=\text{poly}^F_{C_{0}}(E)=\text{poly}^F_{C_{\leq 0}}(E)$.

For $k>0$, observe that $C_k$ and $C_{\leq k-1}$ are disjoint, and that $E=\chi^F_{C_k}(E,E)$. Thus by Proposition \ref{factor1}, we have $\text{poly}^F_{C_{ k}}(E)\cdot\text{poly}^F_{C_{\leq k-1}}(E)=\text{poly}^F_{C_{\leq k}}(E)$. Thus, by induction, we get $\prod_{i=0}^{k}\text{poly}^F_{C_i}(E)=\text{poly}^F_{C_{\leq k}}(E)$.

In the case where $k=n-1$, this gives that $\prod_{C\in\text{Irr}^F(E)}\text{poly}^F_C(E)=\text{poly}^F_{B}(E)=Q^F_E$.
\end{proof}

\begin{proposition}
Let $F=(S,B)$ be a finite factored set, and let $E$ be a nonempty subset of $S$. Then $\text{poly}^F_C(E)$ is irreducible for all $C\in \text{Irr}^F(E)$.
\end{proposition}
\begin{proof}
Let $F=(S,B)$ be a finite factored set, let $E$ be a nonempty subset of $S$, and let $C\in \text{Irr}^F(E)$.

Assume for the purpose of contradiction that $p_0\cdot p_1=\text{poly}^F_C(E)$, and that both $p_0$ and $p_1$ have nonempty support.

By Proposition \ref{factor2}, we have that $p_i=r_i\cdot \text{poly}^F_{C_i}(E)$, for some $r_0,r_1\in\mathbb R$, and $C_0,C_1\subseteq B$. 

We will first need to show that $C_0$ and $C_1$ are nonempty and disjoint. They must be nonempty, because $p_0$ and $p_1$ have nonempty support. Assume for the purpose of contradiction that $b\in C_0\cap C_1$. Let $s$ be an element of $E$, and note that for $i\in\{0,1\}$, we have $[s]_b\in \text{supp }\text{poly}^F_{C_i}(E)$. Thus $[s]_b$ must be degree at least $2$ in $\text{poly}^F_C(E)$, which contradicts the fact that every variable clearly has degree at most $1$ in $\text{poly}^F_C(E)$. 

Next, we need to show that $C_0\cup C_1=C$. We already know that
\begin{equation} \begin{split} \text{supp}(\text{poly}^F_C(E)) & =\text{supp}(r_0r_1\text{poly}^F_{C_0}(E)\text{poly}^F_{C_1}(E)) \\ & = \text{supp}(\text{poly}^F_{C_0}(E))\cup \text{supp}(\text{poly}^F_{C_1}(E)). \end{split} \end{equation}
Let $s$ be an element of $E$. Given an arbitrary $b\in B$, we have that $b\in C$ if and only if $[s]_b\in \text{supp}(\text{poly}^F_C(E))$ if and only if $[s]_b \in\text{supp}(\text{poly}^F_{C_i}(E))$ for some $i\in \{0,1\}$ if and only if $b\in C_0\cup C_1$. 

We now have that $C_0$ and $C_1$ are disjoint and that $C=C_0\cup C_1$. Thus, by Proposition \ref{factor1}, we have that $\text{poly}^F_{C_0}(E)\cdot \text{poly}^F_{C_1}(E)=\text{poly}^F_{C}(\chi^F_{C_0}(E,E))$. Thus $\text{poly}^F_{C}(E)=r_0r_1\text{poly}^F_{C}(\chi^F_{C_0}(E,E))$, so $\text{monos}^F_{C}(E)=\text{monos}^F_{C}(\chi^F_{C_0}(E,E))$.

Let $s_0,s_1\in E$ be arbitrary, and let $s_2=\chi_{C_0}^F(s_0,s_1)$. Note that $\text{mono}_C^F(s_2)\in \text{monos}^F_{C}(\chi^F_{C_0}(E,E))=\text{monos}^F_{C}(E)$, so there is some $s_3\in E$ such that $\text{mono}_C^F(s_2)=\text{mono}_C^F(s_3)$. Thus $s_2\sim_b s_3$ for all $b\in C$. However, we also have that $s_2\sim_b s_1$ for all $b\in B\setminus C$, so $s_2=\chi_C^F(s_3,s_1)$. Since $C\in \text{Irr}^F(E), \chi_C^F(E,E)=E$, so $s_2=\chi_{C_0}^F(s_0,s_1)\in E$. Since $s_0$ and $s_1$ were arbitrary elements of $E$, we have that $\chi^F_{C_0}(E,E)=E$. Since $C_0$ is a nonempty strict subset of $C$, this contradicts the fact that $C\in \text{Irr}^F(E)$.

Thus, $\text{poly}^F_C(E)$ is irreducible for all $C\in \text{Irr}^F(E)$.
\end{proof}

\subsection{Characteristic Polynomials and Orthogonality}
We can now give an alternate characterization of conditional orthogonality in terms of divisibility of characteristic polynomials.
\begin{lemma2}\label{CPO}
Let $F=(S,B)$ be a finite factored set, and let $X,Y,Z\in \parts$ be partitions of $S$. The following are equivalent.
\begin{enumerate}
    \item $\co{X}{Y}{Z}$.
    \item $Q^F_{z}$ divides $Q^F_{x\cap z} \cdot Q^F_{y\cap z}$ for all $x\in X$, $y\in Y$, and $z\in Z$.
    \item $Q^F_{z} \cdot Q^F_{x\cap y\cap z}=Q^F_{x\cap z} \cdot Q^F_{y\cap z}$ for all $x\in X$, $y\in Y$, and $z\in Z$.
\end{enumerate}
\end{lemma2}
\begin{proof}
Clearly condition 3 implies condition 2. We will first show that condition 1 implies condition 3, and then show that condition 2 implies condition 1. 

Let $F=(S,B)$, and let $X,Y,Z\in \parts$ satisfy $\co{X}{Y}{Z}$. Consider an arbitrary $x\in X$, $y\in Y$, and $z\in Z$. We want to show that $Q^F_{z} \cdot Q^F_{x\cap y\cap z}=Q^F_{x\cap z} \cdot Q^F_{y\cap z}$. 

Let $C=h^F(X|z)$. Clearly $C\vdash^F X|z$. We thus have that  $\chi^F_C(z,z)=z$, so $\chi^F_{B\setminus C}(z,z)=z$. We also have that $h^F(Y|z)\subseteq B\setminus C$, so $Y|z\leq_z(\bigvee_S(B\setminus C))|z$.

These two together give that $B\setminus C\vdash^F Y|z$.

Since $C\vdash^F X|z$, we have that $\chi^F_C(x\cap z,z)=x\cap z$. Thus, by Proposition \ref{factor1}, we have that $\text{poly}^F_C(x\cap z)\cdot \text{poly}^F_{B\setminus C}(z)=Q^F_{x\cap z}$. Similarly, since $B\setminus C\vdash^F Y|z$, we have that $\text{poly}^F_C(z)\cdot \text{poly}^F_{B\setminus C}(y\cap z)=Q^F_{y\cap z}$.

Since $\chi^F_C(x\cap z,y\cap z)\subseteq \chi^F_C(x\cap z, z)=x\cap z$, and $\chi^F_C(x\cap z,y\cap z)\subseteq \chi^F_C(z,y\cap z)=y\cap z$, we have $\chi^F_C(x\cap z,y\cap z)\subseteq x\cap y\cap z$. We also have that
\begin{equation} \begin{split} \chi^F_C(x\cap z,y\cap z) & \supseteq \chi^F_C(x\cap y\cap z,x\cap y\cap z) \\ & \supseteq x\cap y\cap z. \end{split} \end{equation}
Thus $\chi^F_C(x\cap z,y\cap z)= x\cap y\cap z$.

By Proposition \ref{factor1}, this gives that $\text{poly}^F_C(x\cap z)\cdot \text{poly}^F_{B\setminus C}(y\cap z)=Q^F_{x\cap y\cap z}$.

Finally, since $\chi^F_C(z,z)=z$, we have that $\text{poly}^F_C(z)\cdot \text{poly}^F_{B\setminus C}(z)=Q^F_{z}$.

Thus, $Q^F_{z}\cdot Q^F_{x\cap y\cap z}$ and $Q^F_{x\cap z} \cdot Q^F_{y\cap z}$ are both equal to $\text{poly}^F_C(x\cap z)\cdot \text{poly}^F_{B\setminus C}(y\cap z)\cdot\text{poly}^F_C(z)\cdot \text{poly}^F_{B\setminus C}(z)$.
%Updated -JR

Thus, condition 1 implies condition 3. It remains to show that condition 2 implies condition 1. 

Fix $F=(S,B)$, and $X,Y,Z\in \parts$, and let $Q^F_z$ divide $Q^F_{x\cap z} \cdot Q^F_{y\cap z}$ for all $x\in X$, $y\in Y$, and $z\in Z$. Assume for the purpose of contradiction that it is not the case that $\co{X}{Y}{Z}$. Thus, there exists some $z\in Z$ such that $h^F(X|z)\cap h^F(Y|z)\neq \{\}$. Let $z\in Z$ and $b\in B$ satisfy $b\in h^F(X|z)\cap h^F(Y|z)$. 

Let $C\subseteq B$ be such that $b\in C$ and $C\in \text{Irr}^F(z)$, and let $p=\text{poly}^F_C(z)$. Thus, $p$ is an irreducible factor of $Q^F_z$.

Either $p$ divides $Q^F_{x\cap z}$ for all $x\in X$ or $p$ divides $Q^F_{y\cap z}$ for all $y\in I$, since otherwise there would exist an $x\in X$ and a $y\in Y$ such that $p$ divides neither $Q^F_{x\cap z}$ nor $Q^F_{x\cap z}$, but does divide their product, contradicting the fact that $p$ is irreducible, and thus prime.

Assume without loss of generality that $p$ divides $Q^F_{x\cap z}$ for all $x\in X$. Fix an $x\in X$. Let us first restrict attention to the case where $x\cap z$ is nonempty. 

Let $Q^F_{x\cap z}=p\cdot q$. By Proposition \ref{factor2},  $p=r_0\cdot \text{poly}^F_{C_0}(x\cap z)$ and $q=r_1\cdot \text{poly}^F_{C_1}(x\cap z)$ for some $r_0,r_1\in \mathbb R$ and $C_0,C_1\subseteq B$. We will show that $C_0=C$, $C_1=B\setminus C$, and $r_0=r_1=1$.

 Let $s$ be an element of $x\cap z$. Then for all $b\in B$, $b\in C$ if and only if $[s]_b\in \text{supp}(p)$ if and only if $[s]_b\in \text{supp}(\text{poly}^F_{C_0}(x\cap z))$ if and only if $b\in C_0$. Thus $C_0=C$. 

For all $b\in B\setminus C$, we have $[s]_b\in \text{supp}(Q^F_{x\cap z})$ and $[s]_b\notin \text{supp}(p)$, so $[s]_b\in \text{supp}(q)$, so $b\in C_1$. Similarly, for all $b\in C_1$, $[s]_b\in \text{supp}(q)$, so $[s]_b\notin \text{supp}(p)$, so $b\in B\setminus C$. Thus $C_1=B\setminus C$.

Since $p$ and $\text{poly}^F_{C_0}(x\cap z)$ both have all coefficients equal to $1$, we have $r_0=1$. Thus, $p=\text{poly}^F_{C}(x\cap z)$.

Similarly, since all the coefficients of $p$ are $1$ and all the coefficients of $Q^F_{x\cap z}$ are $1$, all the coefficients of $q$ are $1$, so $r_1=1$. Thus, $q=\text{poly}^F_{B\setminus C}(x\cap z)$.

We thus have that $Q^F_{x\cap z}= \text{poly}^F_{C}(z)\cdot \text{poly}^F_{B\setminus C}(x\cap z)$.

In the case where $x\cap z$ is empty, we also have $Q^F_{x\cap z}= \text{poly}^F_{C}(z)\cdot \text{poly}^F_{B\setminus C}(x\cap z)$, since both sides are 0.

By Proposition \ref{factor1}, $Q^F_{x\cap z}=\text{poly}^F_{B}(\chi^F_{C}(z,x\cap z))$. Thus, $\text{monos}^F_B(x\cap z)=\text{monos}^F_B(\chi^F_{C}(z,x\cap z))$, so $x\cap z=\chi^F_{C}(z,x\cap z)=\chi^F_{B\setminus C}(x\cap z,z)$.

Since $x\cap z=\chi^F_{B\setminus C}(x\cap z,z)$ for all $x\in X$, we have that $B\setminus C\vdash^F X|z$. However, this contradicts the fact that $b\notin B\setminus C$, and $b\in h^F(X|z)$.

Thus, condition 2 implies condition 1.
\end{proof}
\subsection{Probability Distributions on Finite Factored Sets}
The primary purpose of all this discussion of characteristic polynomials has been to build up to thinking about the relationship between orthogonality and probabilistic independence. We will now discuss probability distributions on finite factored sets.

Recall the definition of a probability distribution.
\begin{definition}
Given a finite set $S$, a probability distribution on $S$ is a function $P:\mathcal{P}(S)\rightarrow\mathbb{R}$ such that \begin{enumerate}
    \item $P(E)\geq 0$ for all $E\subseteq S$,
    \item $P(\{\})=0$,
    \item $P(S)=1$, and
    \item $P(E_0\cup E_1)=P(E_0)+P(E_1)$ whenever $E_0,E_1\subseteq S$ satisfy $E_0\cap E_1=\{\}$.
\end{enumerate}
\end{definition}
A probability distribution on a finite factored set $F$ is a probability distribution on its underlying set that also satisfies another condition, which represents the probability distribution coming from a product of distributions on the underlying factors.
\begin{definition}
Given a finite factored set $F=(S,B)$, a probability distribution on $F$ is a probability distribution $P$ on $S$ such that for all $s\in S$, we have $P(\{s\})=\prod_{b\in B}P([s]_b)$.
\end{definition}

\begin{proposition}
Given a finite factored set $F=(S,B)$, a probability distribution on $S$ is a probability distribution $P$ on $F$ if and only if $P(E)=Q^F_E(P)$ for all $E\subseteq S$.
\end{proposition}
\begin{proof}
If $P(E)=Q^F_E(P)$ for all $E\subseteq S$, in particular this means that $P(\{s\})=Q^F_{\{s\}}(P)=(\prod_{b\in B}[s]_b)(P)=\prod_{b\in B}P([s]_b)$ for all $s\in S$.

Conversely, if $P(\{s\})=\prod_{b\in B}P([s]_b)$ for all $s\in S$, then for all $E\subseteq S$,  $P(E)=\sum_{s\in E} \prod_{b\in B}P([s]_b)=(\sum_{s\in E} \prod_{b\in B}[s]_b)(P)=Q^F_E(P)$.
\end{proof}
\subsection{The Fundamental Theorem of Finite Factored Sets}
We are now ready to state and prove the fundamental theorem of finite factored sets.
\begin{theorem}\label{ft}
Let $F=(S,B)$ be a finite factored set, and let $X,Y,Z\in \parts$ be partitions of $S$. Then $\co{X}{Y}{Z}$ if and only if for all probability distributions $P$ on $F$ and all $x\in X$, $y\in Y$, and $z\in Z$, we have $P(x\cap z)\cdot P(y\cap z)= P(x\cap y\cap z)\cdot P(z)$.
\end{theorem}
\begin{proof}
We already have by Lemma \ref{CPO} that if $\co{X}{Y}{Z}$, then for all $x\in X$, $y\in Y$, and $z\in Z$, $Q^F_{z} \cdot Q^F_{x\cap y\cap z}=Q^F_{x\cap z} \cdot Q^F_{y\cap z}$. Thus for any probability distribution $P$ on $F$, we have 
\begin{equation}
\begin{split}
P(z)\cdot P(x\cap y\cap z)
& = Q^F_{z}(P) \cdot Q^F_{x\cap y\cap z}(P)\\
 & = Q^F_{x\cap z}(P) \cdot Q^F_{y\cap z}(P)\\
  & = P(x\cap z)\cdot P(y\cap z).
\end{split}
\end{equation}
Conversely, assume that for all probability distributions $P$ on $F$, and all $x\in X$, $y\in Y$, and $z\in Z$, we have $P(x\cap z)\cdot P(y\cap z)= P(x\cap y\cap z)\cdot P(z)$. 

If $S$ is empty, then $\{\}$ is the unique partition of $S$, and we have $\co{\{\}}{\{\}}{\{\}}$. Thus, we can restrict our attention to the case where $S$ is nonempty. 

Fix an arbitrary $x\in X$, $y\in Y$, and $z\in Z$. Let $q=Q^F_{x\cap z}\cdot Q^F_{y\cap z}-Q^F_{x\cap y\cap z}\cdot Q^F_{z}$. We will first show that $q(f)=0$ for all $f:\mathcal{P}(S)\rightarrow \mathbb{R}^{>0}$.

Given an arbitrary $f:\mathcal{P}(S)\rightarrow \mathbb{R}^{>0}$, we can define $P_f:\mathcal{P}(S)\rightarrow \mathbb{R}$ by $P_f(E)=Q^F_E(f)/Q^F_S(f)$, and we will show that $P_f$ is a distribution on $F$.

$P_f$ is well-defined because $Q^F_S(f)$ is a nonempty sum of products of positive real numbers, and thus positive. Further, since $Q^F_E(f)$ is a sum of products of positive real numbers, $P_f(E)\geq 0$ for all $E\subseteq S$. Since $Q^F_{\{\}}=0$, we also have $P_f(\{\})=0$. Clearly $P_f(S)=1$. Finally, for all $E_0,E_1\subseteq S$ with $E_0\cap E_1=\{\}$, we have
\begin{equation}
\begin{split}
P_f(E_0\cup E_1)
& = Q^F_{E_0\cup E_1}(f)/Q^F_{S}(f)\\
 & = (Q^F_{E_0}(f)+Q^F_{E_1}(f))/Q^F_{S}(f)\\
  & = P_f(E_0)+P_f(E_1).
\end{split}
\end{equation}
Therefore $P_f$ is a distribution on $S$. We still need to show that $P_f$ is a distribution on $F$.

Observe that for all $s\in S$ and $b\in B$, since $\chi^F_{\{b\}}([s]_b,S)=[s]_b$, we have that $Q^F_{[s]_b}(f)=\text{poly}_{\{b\}}^F([s]_b)\cdot\text{poly}_{B\setminus\{b\}}^F(S)$, and since $\chi^F_{\{b\}}(S,S)=S$, we have that $Q^F_{S}(f)=\text{poly}_{\{b\}}^F(S)\cdot\text{poly}_{B\setminus\{b\}}^F(S)$. Thus, we have that
\begin{equation}
\begin{split}
P_f([s]_b)
& = \text{poly}_{\{b\}}^F([s]_b)(f)/\text{poly}_{\{b\}}^F(S)(f)\\
 & = f([s]_b)/\text{poly}_{\{b\}}^F(S)(f).
\end{split}
\end{equation}
Thus, for all $s\in S$,
\begin{equation}
\begin{split}
\prod_{b\in B}P_f([s]_b)
& = (\prod_{b\in B}f([s]_b))/(\prod_{b\in B}\text{poly}_{\{b\}}^F(S)(f))\\[4pt]
 & = Q^F_{\{s\}}(f)/Q^F_S(f)\\[4pt]
 & = P_f(\{s\}).
\end{split}
\end{equation}
Thus $P_f$ is a distribution on $F$.

It follows that $P_f(x\cap z)\cdot P_f(y\cap z)= P_f(x\cap y\cap z)\cdot P_f(z)$. We therefore have that
\begin{equation}
\begin{split}
q(f) & = Q^F_{x\cap z}(f)\cdot Q^F_{y\cap z}(f)-Q^F_{x\cap y\cap z}(f)\cdot Q^F_{z}(f)\\
 & = (P_f(x\cap z)\cdot P_f(y\cap z)-P_f(x\cap y\cap z)\cdot P_f(z))\cdot Q^F_{S}(f)^2\\
  & = 0\cdot Q^F_{S}(f)^2\\
    & = 0.
\end{split}
\end{equation}
Thus, $q$ is a polynomial that is zero on an open subset of inputs, so $q$ is the zero polynomial. Thus $Q^F_{x\cap z} \cdot Q^F_{y\cap z}-Q^F_{z} \cdot Q^F_{x\cap y\cap z}=0$, so $Q^F_{z} \cdot Q^F_{x\cap y\cap z}=Q^F_{x\cap z} \cdot Q^F_{y\cap z}$. Since $x\in X$, $y\in Y$, and $z\in Z$ were arbitrary, by Lemma \ref{CPO}, we have $\co{X}{Y}{Z}$.
\end{proof}
\section{Inferring Time}\label{inftime}
The fundamental theorem tells us that (conditional) orthogonality data can be inferred from probabilistic data. Thus, if we can infer temporal data from orthogonality data, we will be able to combine these to infer temporal data purely from probabilistic data.

In this section, we will discuss the problem of inferring temporal data from orthogonality data, mostly by going through a couple of examples.
\subsection{Factored Set Models}

We'll begin with a sample space, $\Omega$.

Naively, one might except that temporal inference in this paradigm involves inferring a factorization of $\Omega$. What we'll actually be doing, however, is inferring a factored set \emph{model} of $\Omega$. This will allow for the possibility that some situations are distinct without being distinct in $\Omega$---that there can be latent structure not represented in $\Omega$.

\begin{definition}[model]
Given a set $\Omega$, a model of $\Omega$ is a pair $M=(F,f)$, where $F$ is a finite factored set and $f:\text{set}(F)\rightarrow \Omega$ is a function from the set of $F$ to $\Omega$.
\end{definition}
\begin{definition}
Let $S$ and $\Omega$ be sets, and let $f:S\rightarrow \Omega$ be a function from $S$ to $\Omega$. 

Given a $\omega\in \Omega$, we let $f^{-1}(\omega)=\{s\in S\mid f(s)=\omega\}$. 

Given an $E\subseteq \Omega$, we let $f^{-1}(E)=\{s\in S\mid f(s)\in E\}$. 

Given an $X\in \parts[\Omega]$, we let $f^{-1}(X)\in\parts$ be given by $f^{-1}(X)=\{f^{-1}(x)|x\in X, f^{-1}(x)\neq \{\}\}$.
\end{definition}

\begin{definition}[orthogonality database]
Given a set $\Omega$, an orthogonality database on $\Omega$ is a pair $D=(O,N)$, where $O$ and $N$ are both subsets of $\parts[\Omega]\times \parts[\Omega]\times\parts[\Omega]$.
\end{definition}

\begin{definition}
Given an orthogonality database $D=(O,N)$ on a set $\Omega$, and partitions $X,Y,Z\in \parts[\Omega]$, we write $\cod{X}{Y}{Z}$ if $(X,Y,Z)\in O$, and we write $\ncod{X}{Y}{Z}$ if $(X,Y,Z)\in N$.
\end{definition}

\begin{definition}
Given a set $\Omega$, a model $M=(F,f)$ of $\Omega$, and an orthogonality database $D=(O,N)$ on $\Omega$, we say $M$ models $D$ if for all $X,Y,Z\in\parts[\Omega]$,
\begin{enumerate}
    \item if $\cod{X}{Y}{Z}$ then $\co{f^{-1}(X)}{f^{-1}(Y)}{f^{-1}(Z)}$, and
    \item if $\ncod{X}{Y}{Z}$ then $\neg(\co{f^{-1}(X)}{f^{-1}(Y)}{f^{-1}(Z)})$.
\end{enumerate}
\end{definition}

\begin{definition}
An orthogonality database $D$ on a set $\Omega$ is called consistent if there exists a model $M$ of $\Omega$ such that $M$ models $D$. 
\end{definition}

\begin{definition}
An orthogonality database $D$ on a set $\Omega$ is called complete if for all $X,Y,Z\in\parts[\Omega]$, either $\cod{X}{Y}{Z}$ or $\ncod{X}{Y}{Z}$.
\end{definition}

\begin{definition}
Given a set $\Omega$, an orthogonality database $D$ on $\Omega$, and $X,Y\in\parts[\Omega]$, we say $X<_DY$ if for all models $(F,f)$ of $\Omega$ that model $D$, we have $f^{-1}(X)<^Ff^{-1}(Y)$. 
\end{definition}

\subsection{Examples}
\begin{example}\label{ex1}
Let $\Omega=\{00,01,10,11\}$ be the set of all bit strings of length $2$. For $i\in\{0,1\}$, let $x_i=\{i0,i1\}$ be the event that the first bit is $i$, and let $y_i=\{0i,1i\}$ be the event that the second bit is $i$. Let $X=\{x_0,x_1\}$ and let $Y=\{y_0,y_1\}$. 

Let $v_0=\{00,11\}$ be the event that the two bits are equal, let $v_1=\{01,10\}$ be the event that the two bits are unequal, and let $V=\{v_0,v_1\}$. 

Let $D=(O,N)$, where $O=\{(X,V,\{\Omega\})\}$ and $N=\{(V,V,\{\Omega\})\}$.
\end{example}

\begin{proposition}
In Example \ref{ex1}, $D$ is consistent.
\end{proposition}
\begin{proof}
First observe that $F=(\Omega,\{X,V\})$ is a factored set, and so $M=(F,f)$ is a model of $\Omega$, where $f$ is the identity on $\Omega$. It suffices to show that $M$ models $D$.

Indeed $h^F(X)=\{X\}$, and $h^F(V)=\{V\}$, so $\ortho{X}{V}$, so $\co{f^{-1}(X)}{f^{-1}(V)}{f^{-1}(\{\Omega\})}$. 

Further, it is not the case that $\ortho{V}{V}$, since $V\neq \text{Ind}_\Omega$. Thus it is not the case that $\co{f^{-1}(V)}{f^{-1}(V)}{f^{-1}(\{\Omega\})}$.

Thus $M$ satisfies all of the conditions to model $D$, so $D$ is consistent.
\end{proof}
\begin{proposition}\label{exth1}
In Example \ref{ex1}, $X<_D Y$.
\end{proposition}
\begin{proof}
Let $(F,f)$ be any model of $\Omega$ that models $D$. Let $F=(S,B)$. For any $A\in\parts[\Omega]$, let $H_A=h^F(f^{-1}(A))$. Our goal is to show that $H_X$ is a strict subset of $H_Y$.

First observe that $X\leq_\Omega Y\vee_\Omega V$, so for any $s,t\in S$, if $s\sim_{f^{-1}(Y)}t$ and $s\sim_{f^{-1}(V)}t$, then $f(s)\sim_Y f(t)$ and $f(s)\sim_V f(t)$, so $f(s)\sim_X f(t)$, so $s\sim_{f^{-1}(X)}t$. Thus $f^{-1}(X)\leq_S f^{-1}(Y) \vee_S f^{-1}(V)$.

It follows that $H_X\subseteq h^F(f^{-1}(Y) \vee_S f^{-1}(V))= H_Y\cap H_V$. However, since $\cod{X}{V}{\{\Omega\}}$, we have that $H_X\cap H_V=\{\}$, so 
$H_X\subseteq H_Y$.

By swapping $X$ and $V$ in the argument above, we also get that $H_V\subseteq H_Y$. Since $\ncod{V}{V}{\{\Omega\}}$, we have that $H_V\neq \{\}$. Thus $H_V$ contains some element $b$. Observe that $b\notin H_X$, but $b\in H_Y$. Thus $H_X$ is a strict subset of $H_Y$, so $f^{-1}(X)<^F f^{-1}(Y)$.

Since $(F,f)$ was an arbitrary model of $\Omega$ that models $D$, this implies that $X<_D Y$.
\end{proof}
\begin{example}\label{ex2}
Let $\Omega=\{000,001,010,011,100,101,110,111\}$ be the set of all bit strings of length $3$. For $i\in\{0,1\}$, let $x_i=\{i00,i01,i10,i11\}$ be the event that the first bit is $i$, let $y_i=\{0i0,0i1,1i0,1i1\}$ be the event that the second bit is $i$, and let $z_i=\{00i,01i,10i,11i\}$ be the event that the third bit is $i$. Let $X=\{x_0,x_1\}$, let $Y=\{y_0,y_1\}$, and let $Z=\{z_0,z_1\}$. 

Let $v_0=\{000,001,110,111\}$ be the event that the first two bits are equal, let $v_1=\{010,011,100,101\}$ be the event that the first two bits are unequal, and let $V=\{v_0,v_1\}$. 

Let $D=(O,N)$, where $O=\{(X,V,\{\Omega\}),(X,Z,Y),(V,Z,Y)\}$ and $N=\{(X,Z,\{\Omega\}),(V,Z,\{\Omega\}),(Z,Z,Y)\}$.
\end{example}
\begin{proposition}
In Example \ref{ex2}, $D$ is consistent.
\end{proposition}
\begin{proof}
Let $S=\Omega\cup\{00,01,10,11\}$ be the set of all bit strings of length either 2 or 3. 

For $i\in\{0,1\}$, let $x^\prime_i=\{i00,i01,i10,i11,i0,i1\}$ be the event that the first bit is $i$, and let $X^\prime=\{x_0^\prime,x_1^\prime\}$.

For $i\in\{0,1\}$, let $y^\prime_i=\{0i0,0i1,1i0,1i1,0i,1i\}$ be the event that the second bit is $i$, and let $Y^\prime=\{y_0^\prime,y_1^\prime\}$.

Let $v^\prime_0=\{000,001,110,111,00,11\}$ be the event that the first two bits are equal, let $v^\prime_1=\{010,011,100,101,01,10\}$ be the event that the first two bits are unequal, and let $V^\prime=\{v_0^\prime,v_1^\prime\}$.

For $i\in\{0,1\}$, let $z^\prime_i=\{00i,01i,10i,11i\}$ be the event that the third bit exists and is $i$, let $z^\prime_2=\{00,01,10,11\}$ be the event that there are only two bits, and let $Z^\prime=\{z_0^\prime,z_1^\prime,z^\prime_2\}$.

Let $B=\{X^\prime,V^\prime,Z^\prime\}$. Clearly, $(S,B)$ is a finite factored set.

Let $f:S\rightarrow \Omega$ be given by $f(s)=s$ if $s\in \Omega$, $f(00)=000$, $f(01)=011$, $f(10)=100$, and $f(11)=111$, so $f$ copies the last bit on inputs of length 2, and otherwise leaves the bit string alone. We will show that $(F,f)$ models $D$.

First, observe that $f^{-1}(X)=X^\prime$, $f^{-1}(Y)=Y^\prime$, $f^{-1}(V)=V^\prime$, and $f^{-1}(Z)=\{\{000,010,100,110,00,10\},\{001,011,101,111,01,11\}\}$. 

It is easy to verify that $h^F(X^\prime)=\{X^\prime\}$, $h^F(V^\prime)=\{V^\prime\}$, $h^F(Y^\prime)=\{X^\prime,V^\prime\}$, and $h^F(f^{-1}(Z))=B$. From this, we get that $\ortho{X^\prime}{V^\prime}$ holds, but  $\ortho{X^\prime}{f^{-1}(Z)}$ and $\ortho{V^\prime}{f^{-1}(Z)}$ do not hold.

Next, observe that for $i\in\{0,1\}$, $X^\prime|y_i=V^\prime|y_i=\{\{0i0,0i1,0i\},\{1i0,1i1,1i\}\}$. It is easy to verify that $h^F(X^\prime|y_i)=h^F(V^\prime|y_i)=\{X^\prime,V^\prime\}$. 

Also, observe that $f^{-1}(Z)|y_0=\{\{000,100,00,10\},\{001,101\}\}$, and observe that $f^{-1}(Z)|y_1=\{\{010,110\},\{011,111,01,11\}\}$. It is easy to verify that $h^F(f^{-1}(Z)|y_0)=h^F(f^{-1}(Z)|y_1)=\{Z^\prime\}$.

From this, we get that $\co{X^\prime}{f^{-1}(Z)}{Y^\prime}$ and $\co{V^\prime}{f^{-1}(Z)}{Y^\prime}$ hold, and $\co{f^{-1}(Z)}{f^{-1}(Z)}{Y^\prime}$ does not hold. 

Thus, $(F,f)$ models $D$, so $D$ is consistent.
\end{proof}
\begin{proposition}\label{exth2}
In Example \ref{ex2}, $X<_D Y<_D Z$.
\end{proposition}
\begin{proof}
Let $(F,f)$ be any model of $\Omega$ that models $D$. Let $F=(S,B)$. For any $A\in\parts[\Omega]$, let $H_A=h^F(f^{-1}(A))$. Our goal is to show that $H_X$ is a strict subset of $H_Y$ and that $H_Y$ is a strict subset of $H_Z$.

First observe that $X\leq_\Omega Y\vee_\Omega V$, so $f^{-1}(X)\leq_S f^{-1}(Y)\vee f^{-1}(V)$, so $H_X\subseteq H_Y\cup H_V$. Since $\cod{X}{Y}{\{\Omega\}}$, $H_X\cap H_V=\{\}$, so $H_X\subseteq H_Y$. Symmetrically, $H_V\subseteq H_Y$, so $H_X\cup H_V\subseteq H_Y$.

Similarly, $Y\leq_\Omega X\vee_\Omega V$, so $H_Y\subseteq H_X\cup H_V$. Thus $H_Y=H_X\cup H_V$.

We also know that $H_X$ and $H_V$ are nonempty, because $\ncod{X}{Z}{\{\Omega\}}$ and $\ncod{Y}{Z}{\{\Omega\}}$.

Thus $H_X$ is a strict subset of $H_Y$, so $X<_D Y$.

Let $C\subseteq B$ be arbitrary such that $H_X\cap C$ and $H_V\cap (B\setminus C)$ are both nonempty. Fix some $b_X\in H_X\cap C$ and $b_V\in H_V\cap (B\setminus C)$.

Since $b_X\in H_X$, there must exist $s_0,s_1\in S$ such that $s_0\sim_b s_1$ for all $b\in B\setminus\{b_X\}$, but not $s_0\sim_{f^{-1}(X)} s_1$. Thus it is not the case that $f(s_0)\sim_X f(s_1)$. Without loss of generality, assume that $f(s_0)\in x_0$ and $f(s_1)\in x_1$.

Similarly, since $b_V\in H_V$, there must exist $t_0,t_1\in S$ such that $t_0\sim_b t_1$ for all $b\in B\setminus\{b_V\}$, but not $t_0\sim_{f^{-1}(V)} t_1$. Again, without loss of generality, assume that $f(t_0)\in v_0$ and $f(t_1)\in v_1$. 

For $i,j\in \{0,1\}$, let $r_{ij}=\chi^F_{H_X}(s_i,t_j)$.

Next, observe that $r_{ij}\sim_{f^{-1}(X)} s_i$, so  $f(r_{ij})\sim_X f(s_i)\in x_i$, so $f(r_{ij})\in x_i$. Similarly, $f(r_{ij})\in v_j$, so $f(r_{ij})\in x_i\cap v_j$. Thus, if $i=j$, $f(r_{ij})\in y_0$, and if $i\neq j$, $f(r_{ij})\in y_1$.

Further, observe that $\chi^F_C(r_{00},r_{11})=r_{01}$,
since $r_{00}$ and $r_{11}$ agree on all factors other than $b_X$ and $b_V$. In particular, this means that $\chi^F_C(f^{-1}(y_0),f^{-1}(y_0))\neq f^{-1}(y_0)$. Similarly, since $\chi^F_C(r_{01},r_{10})=r_{00}$, we have that $\chi^F_C(f^{-1}(y_1),f^{-1}(y_1))\neq f^{-1}(y_1)$.

We will use this to show that for any $y\in f^{-1}(Y)$ and $A\in \parts[y]$, either $h^F(A)\cap H_Y=\{\}$, or $H_Y\subseteq h^F(A)$. This is because $h^F(A)\vdash^F A$, so $\chi^F_{h^F(A)}(y,y)=y$, so by the above argument, if $h^F(A)\cap H_X$ is nonempty, then $H_V\subseteq h^F(A)$, which since $H_V$ is nonempty means $h^F(A)\cap H_V$ is nonempty, so $H_X\subseteq h^F(A)$, so $H_Y\subseteq h^F(A)$. Symmetrically, we also have that if $h^F(A)\cap H_V$ is nonempty, then $H_Y\subseteq h^F(A)$. Thus, if $h^F(A)\cap H_Y$ is nonempty, then either $h^F(A)\cap H_X$ or $h^F(A)\cap H_V$ is nonempty, so $H_Y\subseteq h^F(A)$.

Note that for any $y\in f^{-1}(Y)$, two of the elements among the four $r_{ij}$ defined above are in $y$, and those two elements are in different parts in $f^{-1}(X)$, so $f^{-1}(X)|y$ has at least two parts, so $h^F(f^{-1}(X)|y)$ is nonempty. However, $h^F(f^{-1}(X)|y)\subseteq h^F(f^{-1}(X)\vee_S f^{-1}(Y))=H_Y$. Thus, $h^F(f^{-1}(X)|y)\cap H_Y\neq\{\}$, so $H_Y\subseteq h^F(f^{-1}(X)|y)$, so $h^F(f^{-1}(X)|y)=H_Y$. Symmetrically, $h^F(f^{-1}(V)|y)=H_Y$.

In particular, this means that $h^F(f^{-1}(Z)|y)\cap H_Y=\{\}$, since $\cod{X}{Z}{Y}$. 

Since $\ncod{X}{Z}{\{\Omega\}}$, there exists some $b_Z\in H_X\cap H_Z$. Since $b_Z\in H_Z$, there exist $u_0,u_1\in S$ such that $u_0\sim_b u_1$ for all $b\in B\setminus \{b_Z\}$, but it is not the case that $u_0\sim_{f^{-1}(Z)} u_1$. Without loss of generality, assume that $f(u_0)\in z_0$ and $f(u_1)\in z_1$. Let $y=[u_0]_{f^{-1}(Y)}$.

Let $b_y$ be an arbitrary element of $H_Y$. Since $b_Y\in H_Y$, there exist $q_0,q_1\in S$ such that $q_0\sim_b q_1$ for all $b\in B\setminus \{b_Y\}$, but it is not the case that $q_0\sim_{f^{-1}(Y)} q_1$. Without loss of generality, assume that $q_0\in y$ and $q_1\notin y$.

Consider $p_0=\chi^F_{H_Y}(q_0,u_0)=\chi^F_{H_Y}(q_0,u_1)$. Since $q_0\in y$, $p_0\in y$. Since $u_0$ is also in $y$,  $\chi^F_{h^F(f^{-1}(Z)|y)}(p_0,u_0)\sim_{f^{-1}(Z)}p_0$. However, since $h^F(f^{-1}(Z)|y)\cap H_Y=\{\}$, we have $\chi^F_{h^F(f^{-1}(Z)|y)}(p_0,u_0)=u_0$, so $u_0\sim_{f^{-1}(Z)}p_0$.

If $u_1$ were in $y$, we would similarly have $u_1\sim_{f^{-1}(Z)}p_0$, which would contradict the fact that it is not the case that $u_0\sim_{f^{-1}(Z)} u_1$. Thus $u_1\notin y$.

Next, consider $p_1=\chi^F_{H_Y}(q_1,u_0)=\chi^F_{H_Y}(q_1,u_1)$. Since $q_1\notin y$, $p_1\notin y$. Since $u_1$ is also not in $y$,  $\chi^F_{h^F(f^{-1}(Z)|(S\setminus y))}(p_1,u_1)\sim_{f^{-1}(Z)}p_1$. However, since $h^F(f^{-1}(Z)|(S\setminus y))\cap H_Y=\{\}$, we have $\chi^F_{h^F(f^{-1}(Z)|(S\setminus y))}(p_1,u_1)=u_1$, so $u_1\sim_{f^{-1}(Z)}p_1$. 

Thus, it is not the case that $p_0\sim_{f^{-1}(Z)} p_1$. However, we constructed $p_0$ and $p_1$ such that $p_0\sim_b p_1$ for all $b\neq b_Y$. Thus $b_Y\in H_Z$. Since $b_Y$ was arbitrary in $H_Y$, we have that  $H_Y\subseteq H_Z$. Finally, we need to show that this subset relation is strict. 

Since $\ncod{Z}{Z}{Y}$, there is some $y$ such that $h^F(f^{-1}(Z)|y)\neq\{\}$. Let $b$ be any element of $h^F(f^{-1}(Z)|y)$. Since $h^F(f^{-1}(Z)|y)\cap H_Y=\{\}$, $b\notin H_Y$. However, $b\in h^F(f^{-1}(Z)|y)\subseteq h^F(f^{-1}(Z)\vee_S f^{-1}(Y))=h_Z\cup H_Y$. Therefore $b\in H_Z$. Thus $H_Y$ is a strict subset of $H_Z$, so $Y<_D Z$.
\end{proof}
\section{Applications, Future Work, and Speculation}
We will now discuss several different applications and directions for future work. We will divide these research directions into three categories: `Inference,' `Infinity,' and `Embedded Agency.'

This section will be much more speculative than the rest of the paper. It is very likely that some of these avenues for research will turn out to be dead ends, and some of the claims made here may not hold up to further investigation.
\subsection{Inference}
\subsubsection{Decidability of Temporal Inference}

In Section \ref{inftime}, we described a combinatorial problem of inferring temporal relations from an orthogonality database. However, it is not clear whether the question ``Does a given temporal relation follow from a given orthogonality database?'' is decidable.

However, it is not clear whether or not it is decidable whether a given temporal relation follows from a given orthogonality database. 

One way we could hope to decide whether a temporal relation follows from some orthogonality database $D$ over $\Omega$ would be to simply check all factored set models of $\Omega$ that model $D$ up to a given size, and see whether the temporal relation always holds. For this to work, we would need an upper bound on the size of factored sets that we need to consider, as a function of the size of $\Omega$. (Note that the existence of such a bound would not mean that there are no models larger than this upper bound. Rather, it would mean that every model larger than this will have all of the same temporal relations as some smaller model.)

\subsubsection{Efficient Temporal Inference}
Assuming temporal inference is computable, we would further like to be able to infer temporal relations from an orthogonality database \emph{quickly}.

The naive way to get negative results in temporal inference (i.e., to show that certain temporal relations need not hold) would be to search over the space of models. Without the upper bound discussed above, however, this method would only ever yield negative results.

The naive way to get positive results would be to formalize the kind of reasoning used to prove Propositions \ref{exth1} and \ref{exth2}, and search over proofs of this form. It is unclear whether this method can be made efficient.

Alternatively, we could hope to develop some new results and refine our understanding of temporal inference to the point where an alternative method can be made efficient.

\subsubsection{Temporal Inference from Raw Data and Fewer Ontological Assumptions}
In the Pearlian causal inference paradigm, we can infer temporal relationships from joint distributions on a collection of variables.

In Pearl's paradigm, however, this data is already factored into a collection of variables at the outset. Further, the Pearlian paradigm does not make explicit the assumptions that go into this factorization.

Our paradigm instead starts from a distribution on some set of observably distinct worlds. This approach allows us to make fewer ontological assumptions; we don't need to take for granted a particular way the world should be factored into variables. Thus, one might hope that the factored sets paradigm could be used to infer time or causality more directly from raw probabilistic data.

\subsubsection{Causality, Determinism, and Abstraction}
Another issue with the Pearlian causal inference paradigm is that it does not work well in cases where some of the variables are (partially) deterministic functions of each other. Our paradigm has determinism and abstraction built in, so it can be used to infer time in situations where the Pearlian paradigm might not apply.

\subsubsection{Conceptual Inference}
In Example \ref{ex1}, we can infer that $X<_D Y$. We can think of this fact as being about time. However, we can also think of it as being about which concepts are more natural or fundamental. In that example, $X$ and $V$ were more primitive variables, while $Y$ was a more derived variable that was computed from $X$ and $V$.

Suppose we had a symbol that was either 0 or 1, chosen according to some probability, and was also colored either blue or green, chosen independently according to some other probability. We can reason about this symbol using concepts like color or number. Alternatively, we could define a new concept \emph{bleen} meaning ``the symbol is either blue and 0, or green and 1,'' and \emph{grue}, meaning ``the symbol is either green and 0, or blue and 1,'' and use these two concepts instead (cf. \cite{Goodman:1955}).

We want to say that color and number are in some sense better or more useful concepts, while bleen and grue are less useful. Finite factored sets help give formal content to the idea that color and number are more primitive, while bleen and grue are more derived; and this primitiveness seems to point at part of what it means to be a good concept for the purpose of thinking about the world.

\subsubsection{Inferring Time without Orthogonality}
In this paper, we have focused on inferring time from an orthogonality database. Such a database may have been inferred in turn from observed independence and dependence facts drawn from a probability distribution.

We could instead consider inferring time directly from a probability distribution. Cutting out the orthogonality database in this way could even allow us to infer time from a probability distribution that has no nontrivial conditional independencies at all. 

To see why it might be possible to infer time without any orthogonality, consider a set $\Omega$, and a model of $\Omega$, $(F,f)$, where $F=(S,B)$ has $n$ binary factors, and $|\Omega|>n+1$.

There are $n$ degrees of freedom in an arbitrary probability distribution on $F$, and thus at most $n$ degrees of freedom in a probability distribution $P$ on $\Omega$ that comes from a probability distribution on $F$. However, there are $|\Omega-1|$ degrees of freedom in an arbitrary distribution on $\Omega$.

As such, the probability distribution on $\Omega$ will lie on some surface without full dimension in the space of probability distributions on $\Omega$, which could be used to infer some of the properties of $F$.

However, if $|\Omega|$ is much smaller than $|S|$, and $f$ is chosen at random, it is unlikely that there will be any conditional orthogonality relations on partitions of $\Omega$ at all (other than the trivial conditional orthogonality relations that come from one partition being finer than another).

\subsubsection{Inferring Conditioned Finite Factored Sets}

If we modify the temporal inference definition to instead allow for $f$ to be a partial function from $S$ to $\Omega$, we get a new, weaker model of temporal inference. This can be thought of as allowing for the possibility that our distribution on $\Omega$ passes through some filter that only shows us some of the observably distinct worlds.
\subsection{Infinity}
\subsubsection{The Fundamental Theorem of Finitely Generated Factored Sets}
Throughout this paper, we have assumed finiteness fairly gratuitously. It is likely that many of the results can be extended to arbitrary factored sets. However, this generalization will not be immediate. Indeed, even history is not well-defined on arbitrary factored sets.

One intermediate possibility is to consider finite-dimensional factored sets. In this case, history would be well-defined, but our proof of the fundamental theorem would not directly generalize. However, we conjecture that the finite-dimensional analogue of the fundamental theorem would in fact hold.

\begin{conjecture}
Theorem \ref{ft} can be generalized to finite-dimensional factored sets.
\end{conjecture}

On the other hand, we do not expect the fundamental theorem to generalize to arbitrary factored sets. To see why, consider the following example.

\begin{example}
Let $F=(S,B)$, where $S=\mathcal{P}(\mathbb{N})$, $b_n=\{\{s\in S\mid n\in s\},\{s\in S\mid n\notin s\}\}$, and $B=\{b_n\mid n\in \mathbb{N}\}$. Let $X=\{\{\{\}\},S\setminus \{\{\}\}\}$, and let $Y=\{\{\mathbb{N}\},S\setminus \{\mathbb{N}\}\}$. 
\end{example}

In this example, it seems that in the correct generalization of orthogonality to arbitrary factored sets, we likely want to say that $X$ is not orthogonal to $Y$. However, it also seems like we want to say that in every distribution on $F$, at least one of $\{\{\}\}$ and $\{\mathbb{N}\}$ has probability zero, so this should give a counterexample to the fundamental theorem. 
Even without the fundamental theorem, we believe that orthogonality and time in arbitrary-dimensional factored sets will be important and interesting.

\subsubsection{Orthogonality and Time in Arbitrary Factored Sets}

In the infinite-dimensional case, it is not even clear how we should define orthogonality, time, and conditional orthogonality. There are three main contenders. 

First, we could say that (sub)partitions $X$ and $Y$ are orthogonal if there exist disjoint $C_X,C_Y\subseteq B$ such that $C_X\vdash^F X$ and $C_Y\vdash^F Y$. We could then define time as a closure property on orthogonality.

Second, we could just define the history of a (sub)partition $X$ to be the intersection of all $C\subseteq B$ such that $C\vdash^F X$, and leave the definitions of orthogonality and time alone. This second option has some unintuitive behavior. Consider the following example.

\begin{example}
Let $F=(S,B)$, where $S=\mathcal{P}(\mathbb{N})$, $b_n=\{\{s\in S\mid n\in s\},\{s\in S\mid n\notin s\}\}$, and $B=\{b_n\mid n\in \mathbb{N}\}$. Let $Z=\{\{s\in S\mid |s|<\infty\},\{s\in S\mid |s|=\infty\}\}$. 
\end{example}

In this example, $Z$ is orthogonal to itself according to the second option, in spite of having more than one part. However, it is possible that this is a feature, rather than a bug, since it seems to interact nicely with Kolmogorov's zero–one law \citep{Kolmogorov:1956}.

Third, we could define a way to flatten factored sets by merging some of the factors into their common refinement, and we could say $X$ and $Y$ are orthogonal given $Z$ in $F$ if $X$ and $Y$ are orthogonal given $Z$ in some finite-dimensional flattening of $F$. 

The main difference between the first and third options comes from the case where $Z$ has infinitely many parts. In the third option, we must fix a single finite-dimensional flattening such that $X|z$ and $Y|z$ have disjoint histories for all $z\in Z$.

We are most optimistic about the third option, because we conjecture that it can satisfy the compositional semigraphoid axioms, while the other two options cannot. It is also possible that other options give the compositional semigraphoid axioms for partitions with finitely many parts, but not general partitions.

\subsubsection{Continuity and Physics}
A major reason why we are interested in exploring arbitrary-dimensional factored sets is because it could allow us to talk about continuous time. 

The Pearlian paradigm takes advantage of the parenthood relationship between nodes to make inferences. E.g., the nodes are thought of as probabilistic functions of their parents, and the existence of edges between nodes is a central part of temporal inference. 

In the factored set paradigm, there is no mention of parenthood; instead, $\leq^F$ is both reflexive and transitive, and so can be thought of as an ancestry relation. Further, by working with arbitrary partitions rather than a fixed collection of variables, we allow for ``zooming in'' on our variables.

These two properties together suggest that the factored set paradigm is much closer to being able to talk about continuous time, if the theory can be extended naturally to infinite dimensions. 

As pointed out by \citet{Yudkowsky:2012}, physics looks an awful lot like a continuous analogue of Pearlian causal diagrams. We are thus hopeful that when extended to arbitrary dimensions, factored sets could provide a useful new way of looking at physics.

\subsection{Embedded Agency}
\subsubsection{Embedded Observations}
We can use finite factored sets to build a new way of thinking about observations.

\begin{definition}[observes an event]\label{obse}
Let $F=(S,B)$ be a finite factored set. Let $A$ and $W$ be partitions of $S$, and let $E$ be a subset of $S$. 
Let $X_E$ be the partition of $S$ given by $X_E=\{S\}$ if $E=\{\}$ or $E=S$, and $X_E=\{E,S\setminus E\}$ otherwise.
We say $A$ observes $E$ with respect to $W$ (in $F$) if the following two conditions hold.
\begin{enumerate}
    \item $\ortho{A}{X_E}.$
    \item $\co{A}{W}{S\setminus E}.$
\end{enumerate}
\end{definition}
$A$ can be thought of as an agent, with the different parts in $A$ representing options available to $A$. $E$ represents some fact about the world. $W$ can be thought of as some high-level world model. We will especially think of $W$ as a world model that captures all of the information about the world that the agent cares about.

When we say that $A$ observes $E$, this does not necessarily mean that $E$ holds. Rather, we are saying that $A$ can safely assume that $E$ holds. $A$ can safely make this assumption if it is the case that $A$'s choice can't effect whether $E$ holds, and if, when $E$ does not hold,  $A$'s choice can have no effect on any part of the world that $A$ cares about. This is exactly what is represented by the two conditions in Definition \ref{obse}.

In Drescher's (\citeyear{Drescher:2006}) transparent Newcomb thought experiment, the agent cannot be said to observe the contents of the box, because the first condition in Definition \ref{obse} is violated. In Nesov's (\citeyear{Nesov:2009}) counterfactual mugging thought experiment, the agent cannot be said to observe the result of the coin flip, because the second condition is violated.

We can extend this definition to give a notion of an agent observing a partition rather than an event. 

\begin{definition}[observes a partition]\label{obsp}
Let $F=(S,B)$ be a finite factored set. Let $A$, $W$, and $X$ be partitions of $S$. Let $X=\{x_0,\dots, x_{n-1}\}$. 
We say $A$ observes $X$ with respect to $W$ (in $F$) if $\ortho{A}{X}$ and there exist partitions of $S$, $A_i$ for $i\in \{0,\dots,n-1\}$ such that
\begin{enumerate}
    \item $A=\bigvee_S(\{A_i\mid i\in \{0,\dots,n-1\}\})$.
    \item $\co{A_i}{W}{S\setminus x_i}.$
\end{enumerate}
\end{definition}
Saying that $A$ observes $X$ is roughly saying that $A$ can be divided into subagents, where each subagent observes a different part in $X$.

\subsubsection{Counterfactability}
The factored set paradigm also has some interesting things to say about counterfactuals. The chimera functions can be thought of representing a way of taking counterfactuals.

Given a finite factored set $F=(S,B)$, $C\subseteq B$, and $s,t\in S$, let $X_C=\bigvee_S(C)$.

We can think of $\chi_C^F(s,t)$ as the result of starting with $t$, then performing a counterfactual surgery that changes the value of $X_C$ to match its value in $s$.

Unfortunately, while we can tell this story for $X_C$, we cannot tell the same story for an arbitrary partition of $S$.

\begin{definition}[counterfactability]
Given a finite factored set $F=(S,B)$, a partition $X\in \parts$ is called counterfactable (in $F$) if $X=\bigvee_S(h^F(X))$.
\end{definition}

When a partition $X$ is counterfactable, the chimera function gives a well-defined way to start with an element of $S$, and change it by changing what part in $X$ it is in.

Being counterfactable is rather strong, but we have a weaker notion of relative counterfactability. 

\begin{definition}[relative counterfactability]
Given a finite factored set $F=(S,B)$, a partition $X\in \parts$ is called counterfactable relative to another partition $W\in \parts$ (in $F$) if $\co{\bigvee_S(h^F(X))}{W}{X}$.
\end{definition}

$X$ is counterfactable relative to $W$ if $X$ screens off the history of $X$ from $W$. This means that if we want to counterfact on the value of $X$, we can safely counterfact on the finer partition $\bigvee_S(h^F(X))$. As long as we only care about what part in $W$ the result is in, choices about which subpart in $\bigvee_S(h^F(X))$ to counterfact will not matter, so we can think of counterfacting on the value of $X$ as well-defined up to the partition $W$.

This notion of counterfactability explains why counterfactuals sometimes seem clear, and other times they do not seem well-defined. In the factored set ontology, sometimes partitions are not counterfactable because they are not fine enough to fully specify all the effects of the counterfactual. 

\subsubsection{Cartesian Frames}
The factored set paradigm can be seen as capturing many of the benefits of the Cartesian frame paradigm \citep{Garrabrant:2020}. We have already seen this in part in our discussion of embedded observations. We feel that the factored set paradigm successfully captures a meaningful notion of time, while the Cartesian frame paradigm mostly fails at this goal.

The connection between factored sets and Cartesian frames is rather strong. For example, a 2-dimensional factored set model of a set $W$ is in effect a Cartesian frame over $W$. The only difference is that the factored set model forgets which factor is the agent, and which factor is the environment. When one Cartesian frame over $W$ is a multiplicative subagent of another, we can construct a 3-dimensional factored set model of $W$, with the subagent represented by one of the factors, and the superagent represented by a pair of the factors.

\subsubsection{Unraveling Causal Loops}
Whenever an agent makes a decision, there is a temptation to think of the effects of the decision as causally ``before'' the decision being made. This is because the agent uses its model of the effects as an input when making the decision. This causes a problem, because the effects of the decision can of course also be seen as causally after the decision being made.

On our view, part of what is going on is that there is a distinction between the agent's model of the effects, and the effects themselves. The problem is that the agent's model of the effects is highly entangled with the actual effects, which is why we feel tempted to combine them in the first place.

One way to model this situation is by thinking of the agent's model of the effects as being a coarser version of the actual world state after the decision. It is thus possible for the model of the effects to be before the decision, which is before the effects themselves.

By allowing for some variables to be coarsenings or refinements of other variables, the factored set paradigm possibly gives us the tools to be able to straighten out these causal loops.

\subsubsection{Conditional Time}
We can define conditional time similarly to how we define conditional orthogonality.

\begin{definition}[conditional time]
Given a finite factored set $F=(S,B)$, partitions $X,Y\in \parts$, and $E\subseteq S$, we say that $X$ is before $Y$ given $E$ (in $F$), written $X\leq^F Y \mid E$, if $h^F(X|E)\subseteq h^F(Y|E)$.
\end{definition}

It is not clear if this notion has any important philosophical meaning, but it seems plausible that it does. In particular, this notion could be useful for reasoning about situations where time appears to flow in multiple directions at different levels of description, or under different assumptions. Incorporating conditional time could then be used to flatten some causal loops.

\subsubsection{Logical Causality}
Upon discovering logical induction, one of the first things we considered was the possibility of inferring logical causality using our probabilities on logical sentences \citep{Garrabrant:2016:li}. We considered doing this using the Pearlian paradigm, but it now seems like that approach was doomed to fail, because we had many deterministic relationships between our variables.

The factored set paradigm seems much closer to allowing us to correctly infer logical causality from logical probabilities, but it is still far from ready. 

One major obstacle is that the factored set paradigm does not have a reasonable way to think about the uniform distribution on a four-element set. The independence structure of the uniform distribution on a four-element set is not a compositional semigraphoid, because if we take $X$, $Y$, and $Z$ to be the three partitions that partition the four-element set into two parts of size two, then $X$ is independent of $Y$ and of $Z$, but not independent of the common refinement of $Y$ and $Z$.

Since the uniform distribution on a four-element set will likely (approximately) show up many times in logical induction, it is not clear how to do the causal inference.

\subsubsection{Orthogonality as Simplifying Assumptions for Decisions}
While we largely have been thinking of orthogonality as a property of the world, one could also think of orthogonality as something that an agent assumes to make decisions.

For example, when an agent is looking at a coin that came up heads, the agent might make the assumption that its decision has no effect on the worlds in which the coin came up tails. This assumption might only be approximately true, but part of being an embedded agent is working with approximations. Orthogonality seems like a useful language for some of the simplifying assumptions agents might make.

\subsubsection{Conditional Orthogonality and Abstractions}
Given some complicated structure $X$, one might want to know when a simpler structure $Y$ is a good abstraction for $X$. One desirable property of an abstraction is that $Y$ screens off $X$ from all of the properties of the world that an agent cares about, $W$. In this way, by thinking in terms of $Y$, the agent does not risk missing any important information.

We could also consider weaker notions than this, by taking $W$ to just be that which the agent cares about within a certain context in which the agent is using the abstraction. 

This is all very vague and rough, but the point is that conditional orthogonality seems related to what makes a good abstraction, so being able to talk about conditional orthogonality and abstractions together seems like it could prove useful.

\vspace{13mm}

\noindent
\textbf{Acknowledgments}: My thanks to Alex Appel, Ramana Kumar, Xiaoyu He, Tsvi Benson-Tilsen, Andrew Critch, Sam Eisenstat, Rob Bensinger, and Claire Wang for discussion and feedback on this paper.

\pagebreak 

\printbibliography

\end{document}